\documentclass[preprint,12pt]{elsarticle}

\usepackage{amssymb}
\usepackage{subfigure}
\usepackage{amsmath,amscd}
\usepackage{amsthm}
\usepackage{algorithm}
\usepackage{algorithmicx,algpseudocode}
\usepackage{multirow}

\usepackage{natbib}
\usepackage{hyperref}
\usepackage[makeroom]{cancel}

\newtheorem{thm}{Theorem}[section]
\newtheorem{cor}[thm]{Corollary}
\newtheorem{lem}[thm]{Lemma}
\newtheorem{prop}[thm]{Proposition}

\newtheorem{rem}[thm]{Remark}
\graphicspath{{./figures/}}

\newcommand{\mat}[1]{\mathbf{#1}}
\newcommand{\Real}{\mathbb R}

\newcommand{\tr}[1]{\text{tr}[#1]}

\newcommand{\be}{\begin{equation}}
\newcommand{\ee}{\end{equation}}

\usepackage{color}
\usepackage[normalem]{ulem}
\newcounter{ToDo}
\newcounter{guocomm}
\newcounter{Note}
\definecolor{blue-violet}{rgb}{0.54, 0.17, 0.89}
\definecolor{mygreen}{rgb}{0.0, 0.5, 0.0}
\definecolor{awesome}{rgb}{1.0, 0.13, 0.32}
\definecolor{bostonuniversityred}{rgb}{0.8, 0.0, 0.0}

\journal{Pattern Recognition}
\begin{document}

\begin{frontmatter}





%

\title{Cloud removal Using Atmosphere Model}

\author[1]{Yi Guo\corref{cor1}}
\ead{y.guo@westernsydney.edu.au}
\cortext[cor1]{Corresponding author}
\author[2]{Feng Li}
\author[2]{Zhuo Wang}

\address[1]{Centre for Research in Mathematics and Data Science, School of Computing, Engineering and Mathematics, Western Sydney University, Parramatta, NSW 2150, Australia
}
\address[2]{Qian Xuesen Laboratory of Space Technology, Beijing 100094, China}

\begin{abstract}
Cloud removal is an essential task in remote sensing data analysis. As the image sensors are distant from the earth ground, it is likely that part of the area of interests is covered by cloud. Moreover, the atmosphere in between creates a constant haze layer upon the acquired images. To recover the ground image, we propose to use scattering model for temporal sequence of images of any scene in the framework of low rank and sparse models. We further develop its variant, which is much faster and yet more accurate. To measure the performance of different methods {\em objectively}, we develop a semi-realistic simulation method to produce cloud cover so that various methods can be quantitatively analysed, which enables detailed study of many aspects of cloud removal algorithms, including verifying the effectiveness of proposed models in comparison with the state-of-the-arts, including deep learning models, and addressing the long standing problem of the determination of regularisation parameters. The latter is companioned with  theoretic analysis on the range of the sparsity regularisation parameter and verified  numerically. 
\end{abstract}

%


\end{frontmatter}


\section{Introduction}\label{sec:intro}
In this paper, we concern about the satellites imagery. As the imaging sensors are deployed kilometres above the earth ground, clouds usually appear in the acquired images. The clouds are nuisance for data analysis tasks. It is desirable to remove the cloud totally to recover clean ground scene, which gives rise to cloud removal. Due to the versatility of remote sensing imagery, cloud removal methods have to align to the characteristics of the sensors, for example, multiple channels or single band. Meanwhile, the platform is a decisive factor for the design of the algorithm, for example, the computation limitation and power consumption restriction. Furthermore, the analysis tasks after cloud removal has some influence as well. So one has to consider all possible contributing factors in the modelling process. 

Our data is single band satellites images of the same scene sampled from different time points which are subjected to light to moderate cloud covering randomly at various regions. The aim is to recover images without cloud, i.e. the clear images revealing the ground scene so that subsequent analysis can be performed reliably, for example, object detection and tracking. Therefore the fidelity is the most important factor to be considered, in other words, the recovered must be as close as possible to the truth, not just simply ``visually fit'' (look plausible from afar). Unfortunately, there is no objective assessment except visual checking, and one of the goals of this paper is to fill this gap. 

We focus on non-deep-learning based methods for cloud removal, although latest deep learning methods were used as contenders in our empirical studies subject to code availability, for example \cite{zhengSingleImageCloud2021} and \cite{Sarukkai_2020_WACV}. The reason for this is that the fidelity of the recovered images is a concern for deep learning based methods. The workflow of these methods consists of two steps. The first is to identity cloud covered areas and remove them. The second is to apply generative models to fill the removed pixels. Generalised adversial networks (GAN) based models are popular choice for image completion. However, the working mechanism of GAN and its variants, heavily relies on the training data on which the distribution is modelled by transforming a specified random distribution, e.g. uniform distribution or multivariate Gaussian distribution. Essentially, GAN is some sort of density estimator. Then the question is, what if the scene that the satellite sampled never appears in the training data? GAN will certainly generate something for the missing areas but will not be able to stretch outside its modelled distribution even it is conditioned on some posterior. Therefore we consider other alternatives, for example, temporal mosaicing \cite{guoCloudFilteringLandsat2016b,guoMultipleTemporalMosaicing2017b}. Although enforcing spatial smoothness is the most time consuming component, the fidelity can be reassured that no ``alien pixels'' will be inserted into the images like GAN based methods do. Another possibility is matrix completion methods for missing pixel filling, for example, \cite{wenTwoPassRobustComponent2018} and its later development \cite{zhangCoarsetoFineFrameworkCloud2019}. The main model behind these methods is the low rank robust principal component analysis \cite{CandesLiMaWright2010} coming from a long development of robust PCA (RPCA) \cite{delatorreRobustPrincipalComponent2001,GaoKwanGuo2009} that is the efforts to improve the robustness of the linear PCA model by reducing the sensitivity to outliers. The elegance of RPCA comparing to its peers is the simplicity in its formation as well as its theoretical guarantee for the recovery of the low rank signals and sparse noise. The application of RPCA implies that the observed images are the summation of low rank ground images and sparse cloud cover images (images with cloud only without background). It makes sense for such arrangement assuming that the ground scene changes little after excluding misalignment and geometric distortion, and clouds cover only small portion of the scene.  The low rank condition on ground component signals the way of filling missing pixels and hence RPCA has better interpretability than GAN methods. 

It seems that the aforementioned two-step workflow should be able to be consolidated to a single one using RPCA. Nonetheless this two-step strategy was still adopted for no obvious reason, in which RPCA is only used for cloud identification and a low rank matrix completion follows after those cloud affected areas masked out. Two questions remains though. Firstly,   where is the atmosphere modelled in the image data?
The atmosphere is reflected as a thin haze layer in the acquired images which may not be negligible. Secondly, is the simple additive model in RPCA really the right description of the physics? Apparently not. The most realistic model so far is the so-called atmosphere scattering model \cite{narasimhanVisionAtmosphere2002} for satellite images. Therefore one should build atmospheric affect into the model for cloud removal and ground images recover.

\section{Models considering atmosphere effects}\label{sec:model}
Before presenting proposed ones, we first describe RPCA based methods here in the setting of imagery applications. Let $ I_i \in \Real^{d_1\times d_2}$ be the $i$-th sampled image of size $d_1\times d_2$ and $i=1,\ldots,n$; $D=[vec(I_1), \ldots, vec( I_n)]$ where $vec(X)$ is the vectorisation of matrix $X$ to be a column vector, and hence $ D \in \Real ^{d\times n}$ ($d=d_1d_2$). The RPCA model shared in \cite{wenTwoPassRobustComponent2018,zhangCoarsetoFineFrameworkCloud2019} is the following, 
\begin{align}\label{e:rpca}
\min_{L,C} & \|L\|_* + \lambda \|C\|_1 \\
\text{s.t. }&D = L + C\notag 
\end{align}
where $\|X\|_*$ is the nuclear norm  of $X$, i.e. the summation of all singular values of $X$, which is the convex envelope for matrix rank, $\|X\|_1$ is the $\ell_1$ norm of $X$,  $L$ is the initial recovered ground images, $C$ is the cloud cover images, and both are the same size as $D$. $\lambda$ is the regularisation parameter usually fixed to be $\frac1{\sqrt{d}}$ as recommended in \cite{CandesLiMaWright2010}. 
By introducing group sparsity (defined by super-pixels) and alignment into \eqref{e:rpca},  \cite{zhangCoarsetoFineFrameworkCloud2019} claims slightly better performance. After solving \eqref{e:rpca}, both methods proceed to matrix completion with the mask derived from $C$ as follows 
\begin{align}\label{e:mc}
\min_{B,S} & \|L\|_* +  \alpha\|S_{\Omega}\|_1 + \beta \|S_{\bar\Omega}\|_1 \\
\text{s.t. }&D = B + S\notag 
\end{align}
where $\Omega$ is the mask matrix of size $d\times n$ with 0's for masked out elements and 1's for others, $\bar\Omega$ is the negated version of $\Omega$, i.e. flipping 0's and 1's, and $S_{\Omega}$ is the projection of $S$ on $\Omega$, i.e. masking out elements indicated by 0's in $\Omega$. The $ij$th element in the mask matrix, $[\Omega]_{ij}=1$ if $[C]_{ij}>\gamma\sigma(vec(C))$ and $[\Omega]_{ij}=0$ otherwise, where $\sigma(v)$ is the standard deviation of $v$ and $\gamma\in[0,1]$ is a pre-set ratio. $B$ is the final recovered ground images, which are supposed to be cloud free. $S$ is the noise. In implementation,  $\gamma=0.8$, $\alpha=\frac{0.1}{\sqrt{d}}$ and $\beta=1$. Both problems are convex with two blocks of variables. There are many gradient projection based solvers/optimisers for them under the ADMM framework \cite{BoydVandenberghe2004}. 
They all work reasonably well for moderate size of images, for example, $d_1=d_2=1024$ and $n=7$. 

The critical step is in \eqref{e:rpca} where cloud cover $C$ is supposed to be separated. Note that the decomposition of the observed data $D = L + C$ reflects the basic model assumption. As mentioned earlier, this departures from the reality by ignoring atmosphere effect. So instead of simple additive model we propose to use atmosphere scattering  \cite{narasimhanVisionAtmosphere2002}, $D = L\circ(1-C) + C$, in the modelling, and hence optimising the following
\begin{align}\label{e:atm}
\min_{L,C} & \|L\|_* + \lambda \|C\|_1 \\
\text{s.t. }&D =L\circ(1-C) + C \notag \\
&[L]_{ij}\in[0,1],\ [C]_{ij}\in[0,1]\notag
\end{align}
where $X\circ Y$ is the element-wise product of matrix $X$ and $Y$ of the same size. In the above formulation, it is assumed that the pixels in observed images are rescaled to $[0,1]$, which is easily done by dividing the maximum digital number of the sensor, but {\it not} the maximum of the observed values. Note that \eqref{e:atm} is no longer a convex problem as the equality condition is not affine. It is supposed to be much difficult to solve on itself, let alone the boxed conditions clamping the elements in both $L$ and $C$ within $[0,1]$. Nonetheless, there is still some strategies for the optimisation. Fore example, introducing a dummy variable $X$ to untangle the interaction between $L$ and $C$
\begin{align}\label{e:atm_op1}
\min_{L,C} & \|L\|_* + \lambda \|C\|_1 \\
\text{s.t. }&D =X\circ(1-C) + C\notag \\
& L = X \notag \\
&[L]_{ij}\in[0,1],\ [C]_{ij}\in[0,1],\ [X]_{ij}\in[0,1]\notag
\end{align}
and proceed with the normal ADMM. However, we observed that this does not converge well enough to be practically useful. Instead, we employ linearisation using primal accelerated proximal gradient method \cite{PongTsengJiYe2010} for its ease in handling entangled nuclear norm optimisation and stability. The Lagrange of \eqref{e:atm} with proximity is 
\begin{align}\label{e:atmlagragian1}
\mathcal L = &\|L\|_* + \lambda \|C\|_1 + \langle Y, D - L\circ(1-C) - C \rangle \\
&+\frac\mu2\|D - L\circ(1-C) - C\|_F^2 \notag
\end{align}
leading to 
\be\label{e:atmlagragian}
\mathcal L = \|L\|_* + \lambda \|C\|_1 +\frac\mu2\|D - L\circ(1-C) - C+\frac Y{\mu}\|_F^2
\ee
by ignoring constants, where $\|X\|_F$ is the Frobenius norm of $X$, $Y\in\Real^{d\times n}$ is Lagrangian parameters for the equality condition and $\mu\ge 0$ is the proximity coefficient. Note that  \eqref{e:atmlagragian1} and \eqref{e:atmlagragian} are the proximal form and the boxed conditions in \eqref{e:atm} are ignored at this stage, which will be handled later by feasibility projection after updating all unknowns. Alternating the minimisation w.r.t. $L$ and $C$ is adopted here. Apparently minimising $\mathcal L$ with respect to $L$ is difficult due to the term $L\circ(1-C)$ although no much trouble for $C$. The gradients are shown below. 
\begin{align}
\frac{\partial \mathcal L }{\partial C} = \lambda\partial\|C\|_1 - \mu(D - L\circ(1-C) - C+\frac Y{\mu})\circ(1-L)  \label{e:gradLtoC} \\
\frac{\partial \mathcal L }{\partial L} = \partial\|L\|_* - \mu(D - L\circ(1-C) - C+\frac Y{\mu})\circ(1-C)  \label{e:gradLtoL}
\end{align}
where $\partial\|X\|_1$ and $\partial\|X\|_*$ are subgradients of $\ell_1$ norm and nuclear norm respectively. The stationary point of \eqref{e:gradLtoC} gives closed form solution $C^*$
\be\label{e:updateC}
[C^*]_{ij} = \left\{
\begin{array}{ll}
	0, &b_{ij}\in[-\lambda,\lambda] \\
	\frac{\lambda-|b_{ij}|}{a_{ij}} sign(b_{ij}), &\text{otherwise}
\end{array}
\right.
\ee
where $a_{ij} = \mu[(L-1)\circ(L-1)]_{ij}$, $b_{ij}=\mu[(D-L+\frac Y{\mu})\circ(L-1)]_{ij}$ and $sign(x)$ is the sign function of $x$ which takes 1 when $x>0$ and $-1$ when $x<0$. It is a straightforward soft thresholding for $\ell_1$ norm minimisation. The only difference is the regularisation is not global but local or adaptive as the regularisation parameter $\lambda$ is rescaled by each $1/a_{ij}$ as shown in \eqref{e:updateC}. Whereas there is no closed form solution for $\frac{\partial \mathcal L }{\partial L} = 0$ because the singular value thresholding (SVT)\cite{caiSingularValueThresholding2010} only works for the following general form 
\[
\tau\partial\|X\|_* + X -A = 0
\]
where $\tau>0$ is an arbitrary scaler (normally regularisation parameter) and $A$ is a matrix the size as $X$. The solution $X^*$ to above is $X^*=\mathcal S_{\tau}(A)$ and $\mathcal S_{\tau}(A)$ is the so-called SVT operator defined as 
\be\label{e:svt}
\mathcal S_{\tau}(A) = U\left(
\begin{array}{cccc}
	\max(\sigma_1-\tau,0) & & &\\
	&\ddots & \\
	& & \max(\sigma_n-\tau,0)
\end{array}
\right)V^\top
\ee
where $U$ and $V$ are from SVD of $A$, i.e. $A = U\Sigma V^\top$ and $\Sigma =\text{diag}(\sigma_1,\ldots,\sigma_n)$. 

To work around it, we linearise the smooth part in \eqref{e:atmlagragian}  
\be
\frac\mu2\|D - L\circ(1-C) - C+\frac Y{\mu}\|_F^2\equiv f_s(L) 
\ee
by the first order Taylor expansion with proximal term w.r.t $L^k$, the $k$th value of $L$ in the iterative optimisation for \eqref{e:atmlagragian}, and optimise $L$ while holding other variables constant as
\[
\min_{L} \|L\|_* + \langle  \frac{\partial f_s }{\partial L}|_{L^k}, L - L^k \rangle + \frac{\ell_p}2\|L - L^k\|_F^2.
\]
In above, $ \frac{\partial f_s }{\partial L}|_{L^k} = \mu(D-L^k\circ(1-C)-C+\frac{Y}\mu)\circ(1-C)$. $\ell_p$ is the Lipschitz constant of $f_s(L)$,  which is the operator norm of $ \langle  \frac{\partial f_s }{\partial L}|_{L^k}, \cdot \rangle$ that maps a matrix of the same size of $L$ to $\Real$
\begin{align*}
\langle  \frac{\partial f_s }{\partial L}|_{L^k}, \cdot \rangle:\ \Real^{n\times d}&\rightarrow \Real\\
A&\mapsto \langle  \frac{\partial f_s }{\partial L}|_{L^k}, A \rangle
\end{align*}
It is straightforward to see that $\ell_p=1$ due to the box conditions of $C$ and $L$. This leads to
\be\label{e:linl}
\min_{L} \|L\|_* + \frac{\ell_p}2\|L - L^k + \frac{\partial f_s}{\partial L}|_{L^k}/{\ell_p}\|_F^2
\ee

The above linear approximation results is very convenient as interaction between $L$ and $C$ has been removed and therefore \eqref{e:linl} has closed form solution using SVT. We apply Nesterov acceleration to speed up the process, which is proven to be convergent for \eqref{e:linl} with carefully chosen optimisation parameters \cite{Nesterov2003}. This iterative procedure for $L$ has to be embedded into the optimisation for \eqref{e:atmlagragian} and hence there are two loops in entire algorithm. The detailed optimisation algorithm for solving \eqref{e:atm} is listed in Alg.~\ref{alg:atm}. Note that the boxed conditions are satisfied by clamping in Alg. \ref{alg:atm}, which is the feasibility projection commonly used in many implementations \cite{LiuJiYe2009a}. We call the model in \eqref{e:atm} and its realisation in Alg. \ref{alg:atm} atmosphere cloud removal model and ATM for short.

\begin{algorithm}
	\caption{Solving atmosphere scattering model for cloud removal in \eqref{e:atm}}
	\begin{algorithmic}[1]
		\Require $D$, $\lambda$, $\epsilon$ 
		\State $Y=D/\lambda$, $L=0$, $C=0$
		\State $\rho=1.5$, $\mu=1.25/\|D\|$, $\mu_{\max}=\mu*10^7$
		\While{$\frac{\|D - C - (1-C)\circ L\|_F}{\|D\|_F}>\epsilon$ }
		\State update $C$ using \eqref{e:updateC}	
		\State $C(C<0) = 0$,  $C(C>1) = 1$ \Comment{Feasibility projection for $C$}
		\State $\theta=\underline\theta=1$, $\underline L=L$, $\ell_p = 1$
		\State $\underline f  =0$, $f=\infty$, $k_{\max}=100$
		\While{$k<k_{\max}$ and $|f-\underline f|>10^{-3}$}
		\State $W = L+(\theta/\underline\theta-\theta)(L-\underline L)$
		\State $\underline L=L$, $L=\mathcal S_{\frac1{\mu\ell_p}}(W)$
		\State $\underline\theta=\theta$, $\theta=\frac{\sqrt{\theta^4+4\theta^2}-\theta^2}2$
		\State $\underline f=f$, $f  =\frac\mu2\|D-C-(1-C)\circ L\|_F^2+ \|L\|_*$
		\EndWhile
		\State $L(L<0) = 0$,  $L(L>1) = 1$ \Comment{Feasibility projection for $L$}
		
		\State $Y=Y+	D - C - (1-C)\circ L$
		\State $\mu = \min(\mu\rho, \mu_{\max})$
		\EndWhile
	\end{algorithmic}
	\label{alg:atm}
\end{algorithm}

Due to the iterative procedure for solving $L$ in ATM, it is expected to be slow. However, the recovered cloud component, i.e. $C$ is closer to reality than that from RPCA as shown in Fig. \ref{fig:cloudsdetected}, where the source images are from GaoFen4 satellite captured at the same scene at 7 time points. 
\begin{figure*}
	\begin{center}
		\includegraphics[width=0.3\textwidth]{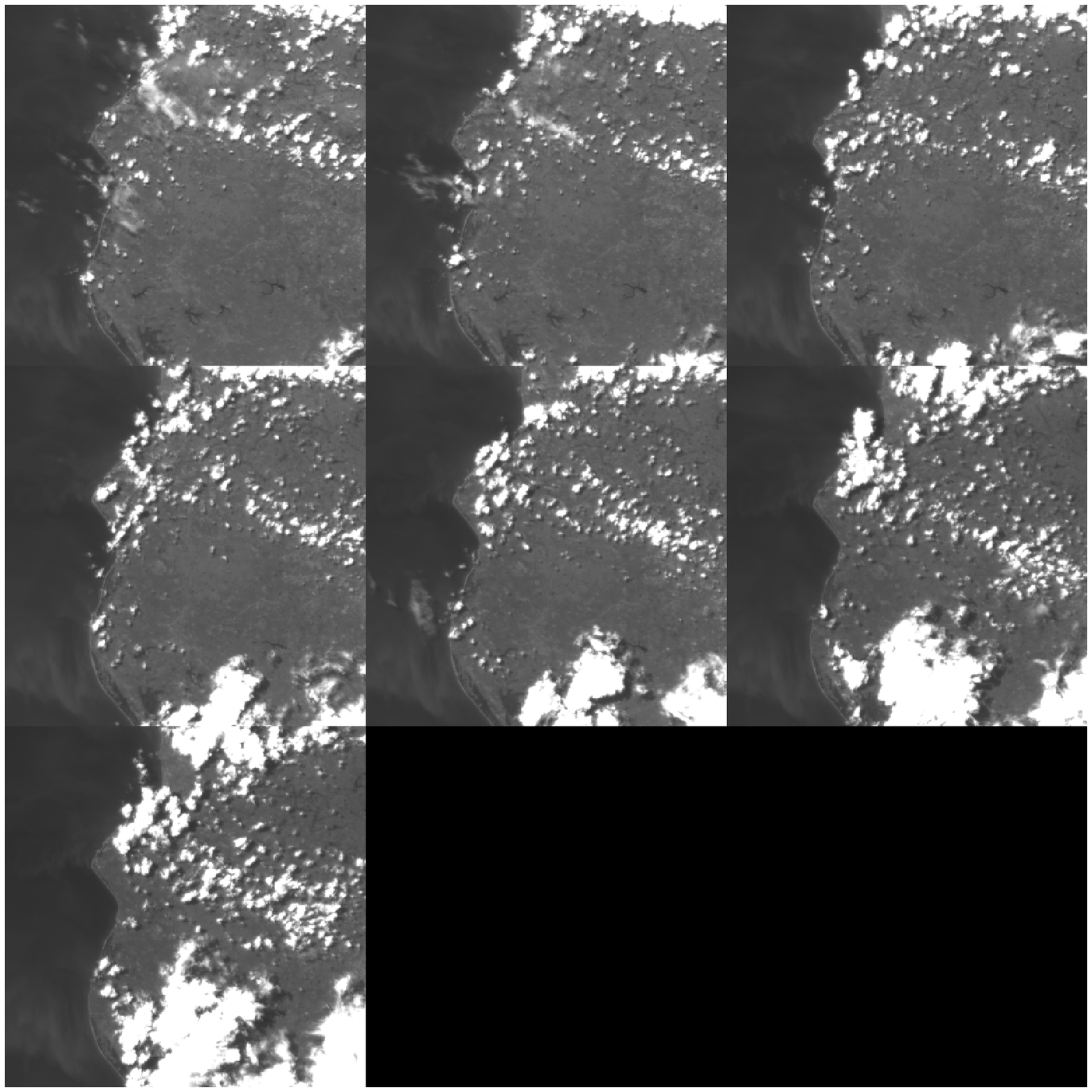}
		\includegraphics[width=0.3\textwidth]{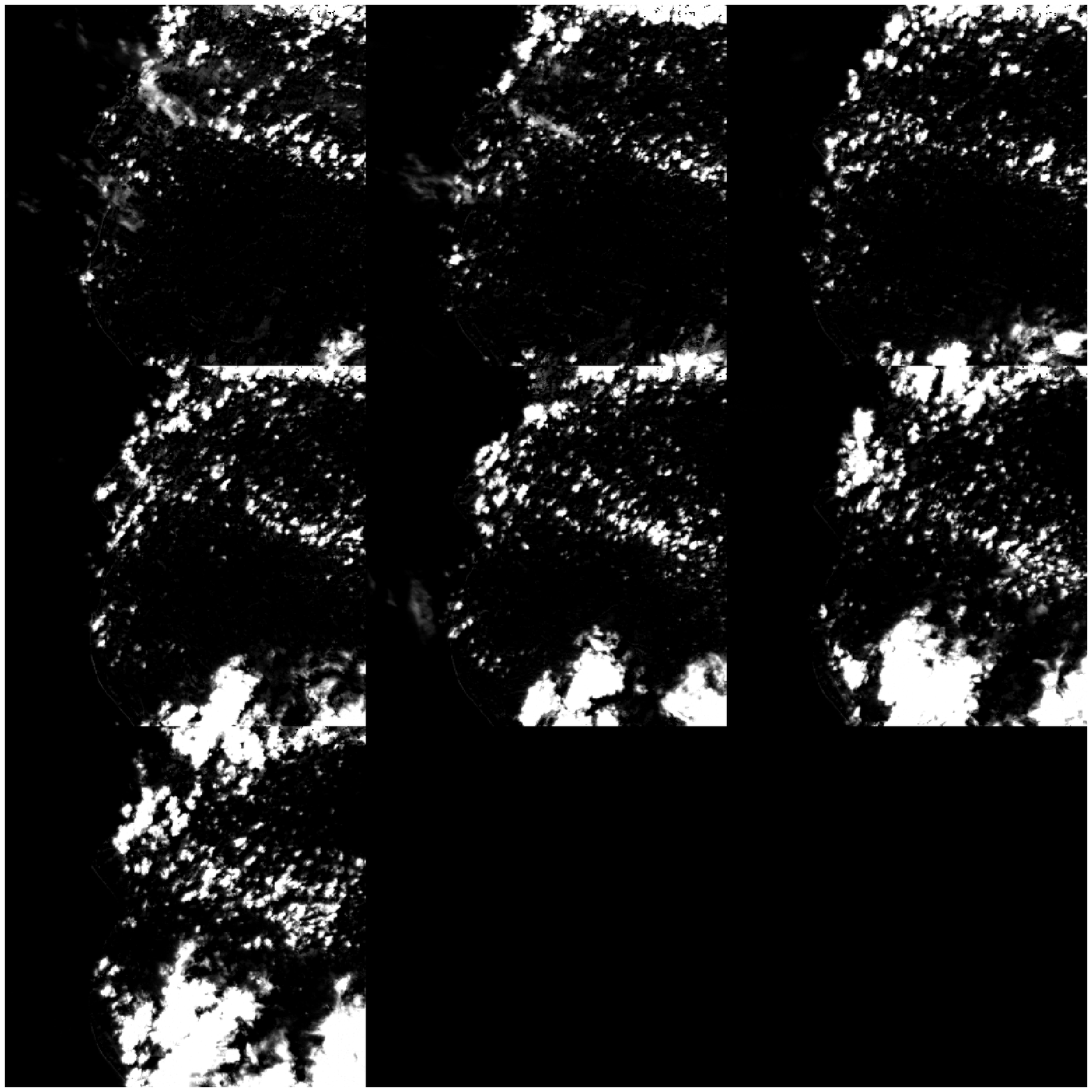}
		\includegraphics[width=0.3\textwidth]{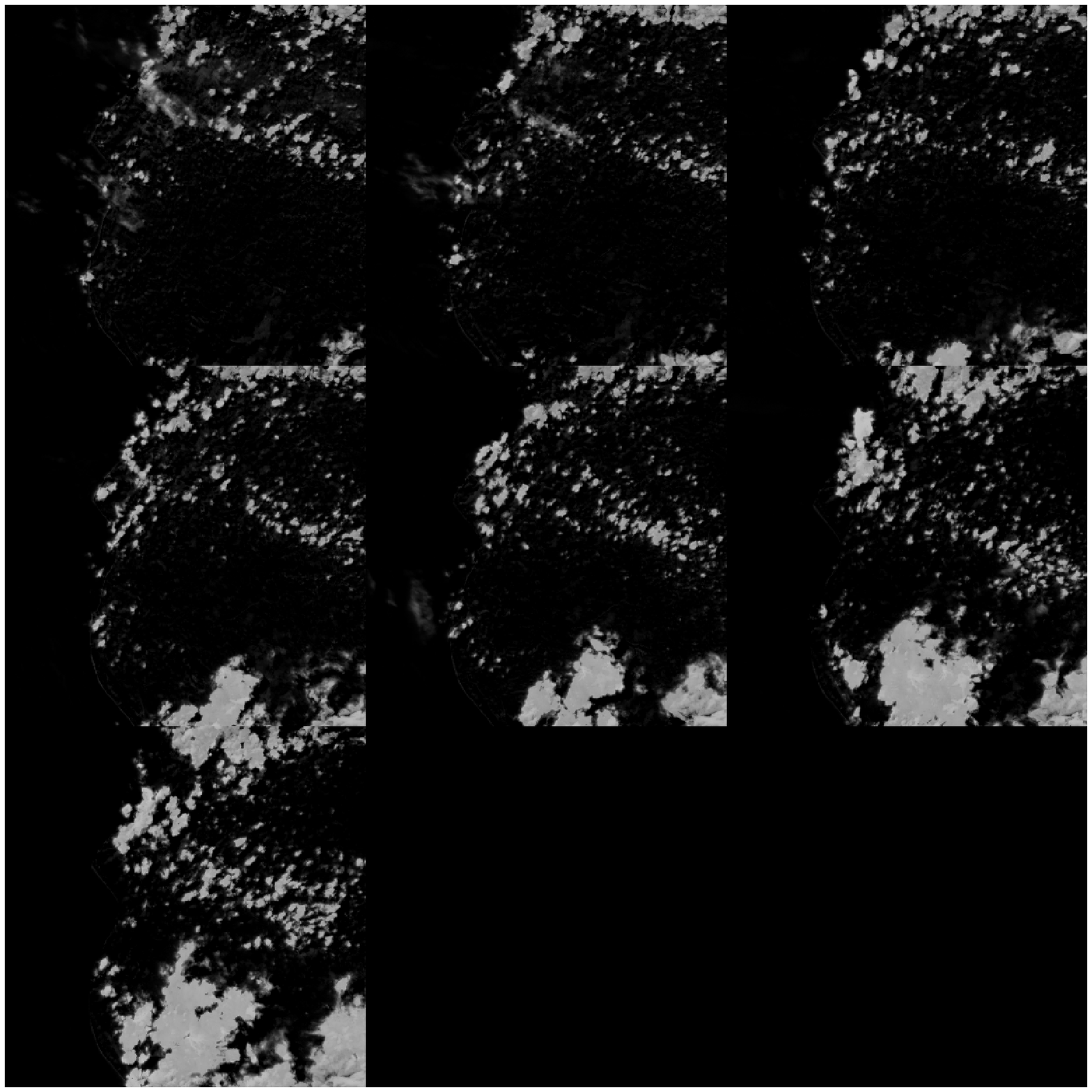}
		\caption{Cloud detection for GaoFen4 images of the same scene at 7 time points. From left: original GF4 images, ATM detected clouds, RPCA detected clouds.}
		\label{fig:cloudsdetected}
	\end{center}
\end{figure*}


It is clear that the ATM detected clouds are much brighter than those detected by RPCA thanks to its detailed atmosphere model, at the cost of much higher computational load as shown in Fig. \ref{fig:sim7time}. This motivates us to reduce its computational cost while maintaining model capacity. The key is to disentangle the interaction between $L$ and $C$ that breaks the convexity. Let us take a closer look the core in ATM model in \eqref{e:atm}, i.e. $D = L\circ (1-C) +C$. We decompose $L$ as $L = (1-P)\circ L + P\circ L$ for $P \in [0,1]$ and $P\succeq C$, where $\succeq$ means element-wise $\ge$, i.e. $[P]_i\ge [C]_i$. We proceed using this decomposition 
\[
D = L + C - L\circ C= (1-P)\circ L + C + (P-C)\circ L. 
\]
Under the choice of $P$, $N\equiv(P-C)\circ L\succeq 0$. The above can be written as 
\begin{equation}
D = \tilde L  + C + N
\end{equation}\label{e:equivatm}
for $\tilde L= (1-P)\circ L$.  We can easily write out an equivalent optimisation problem to \eqref{e:atm} using \eqref{e:equivatm} with many coupling conditions, which complicate the optimisation. However, if we drop some coupling conditions, i.e. relaxation and approximation, it will be much easier to solve, and yet the coupled problem is still a special case of the relaxed version. So we optimise the following
\begin{align}\label{e:atm2}
\min_{L,C,N} & \|L\|_* + \lambda \|C\|_1 + \beta\|N\|_F^2\\
\text{s.t. }&D = L + C + N \notag \\
&[L]_{ij}\in[0,1],\ [C]_{ij}\in[0,1],\ [N]_{ij}\in[0,1]\notag
\end{align}
Note that in above $L$ replace $\tilde L$ which is an approximation. We highlight this is a relaxed version of \eqref{e:atm} with its own interpretation, that is $N$ acts as a thin haze layer accounting for the atmosphere.

In \eqref{e:atm2} the values in $N$ are controlled by the Frobenius norm. It is well known that the Frobenius norm will not encourage sparsity, but compress the values towards zeros uniformly. Depending on the value of $\beta$, $C+N$ can reach the so-called $\alpha$-sparsity \cite{EldarKutyniok2012}, i.e. sparsity beyond value $\alpha$. Note that we fix $\beta=1$ throughout this paper. 

Eq. \eqref{e:atm2} is significantly easier to solve than eq. \eqref{e:atm} for being convex with no  interaction terms in the low rank component. Although direct generalisation of ADMM to more than two blocks of variables like those in \eqref{e:atm2} may not converge as shown in \cite{ChenHeYeYuan2014} with crafted counter examples, from many other applications, and vast amount of experiments we carried out, the optimisation converged quite quickly. Detailed optimisation algorithm is listed in Alg. \ref{alg:atm2}. We call the model in \eqref{e:atm2} and its optimisation algorithm in Alg. \ref{alg:atm2} alternative ATM, or aATM for short. 

\begin{algorithm}
	\caption{Solving \eqref{e:atm2}}
	\begin{algorithmic}
		\Require $D$, $\lambda$, $\beta$, $\epsilon$ 
		\State $Y=D/\lambda$, $L=0$, $C=0$, $N=0$
		\State $\rho=1.5$, $\mu=1.25/\|D\|$, $\mu_{\max}=\mu*10^7$
		\While{$\frac{\|D - C -L -N\|_F}{\|D\|_F}>\epsilon$ }
		\State $[C]_{ij} = sign([C]_{ij})\max(|[C]_{ij}|-\frac\lambda\mu,0)$	
		\State $C(C<0) = 0$,  $C(C>1) = 1$ \Comment{Feasibility projection for $C$}
		\State $L=\mathcal S_{\frac1{\mu}}(D-C-N+\frac{Y}\mu)$
		\State $L(L<0) = 0$,  $L(L>1) = 1$ \Comment{Feasibility projection for $L$}
		\State $N=\frac\mu{(\beta+\mu)(D-L-C+\frac{Y}\mu)}$
		\State $N(N<0) = 0$,  $N(N>1) = 1$ \Comment{Feasibility projection for $N$ but not necessary}
		\State $Y=Y+	D - C - L - N$
		\State $\mu = \min(\mu\rho, \mu_{\max})$
		\EndWhile
	\end{algorithmic}
	\label{alg:atm2}
\end{algorithm}

For the regularisation parameters, we provide theoretic analyse on the range of the main regularisation parameter $\lambda$ in Section \ref{sec:analysis}. The results align with the empirical study outcomes presented in the next section. Furthermore, we will present an empirical equation based on numerical method to determine the best value for $\lambda$ as a guidance for practical use.

We need to point out that all our models can be used directly to recover ground images unlike the main contenders \cite{wenTwoPassRobustComponent2018,zhangCoarsetoFineFrameworkCloud2019} where a matrix completion (MC) step has to follow although it is debatable whether MC is necessary. However, without some sort of ground truth, it would be a myth and the arguments would be meaningless. To address this long standing issue, we design a semi-realistic simulation of cloud covered images so that cloud and ground images are known. 

%
\section{Quantification of performance}\label{sec:exp}
\subsection{Simulation and performance indicator}
Cloud removal experiments are normally conducted on real images from satellites and the evaluation of the effectiveness of the recovery is based on visual checking and cloud cover by IoU (Intersection over Union) originated from computer vision \cite{LigginsChongKadarAlfordVannicolaThomopoulos1997} which is basically Jaccard index \cite{jaccardDISTRIBUTIONFLORAALPINE1912}.  The ground truth of cloud cover is obtained by time consuming manual labelling of clouds. Due to the complexity of the nature of clouds, it is extremely difficult to delineate the boundary of cloud clusters accurately, especially for thin clouds, and hence there exist large amount of errors when segmenting clouds manually. An ideal solution is to build cloud model to capture the shape and formation of all sorts of clouds, thick or thin. Unfortunately it is quite involved in physics and mathematics and it is a multi-facet problem \cite{dobashiUsingMetaballsModeling1999,dobashiVisualSimulationClouds2017, yuanModellingCumulusCloud2014,xingThreedimensionalParticleCloud2017}. Even if the cloud cover is known, the other side of the problem, way more important than cloud, is the ground truth of the ground scene. The ultimate goal of cloud removal is to recover ground scene accurately. Whereas current practice largely relies on subjective evaluation, or ``eye-balling'', which is apparently very vulnerable to bias. Therefore, an objective and robust evaluation is highly desirable. The work in \cite{zhengSingleImageCloud2021} used an overly simplified method to train the Unet for cloud separation by simulating random strips of white rectangles or from brightest to darkest colour gradient boxes on top of clear ground images. This is a bit primitive. Not only are they far from real clouds, but most importantly the regular shape reduces the complexity of the problem. Inspired by the success of applying Perlin noise \cite{perlinImageSynthesizer1985,perlinImprovingNoise2002} in the simulation of  virtual landscapes, we adopt Perlin noise to generate synthetic clouds. We take a cloud free image, say from the Inria aerial image labeling dataset \cite{maggiori2017dataset}, convert  it to greyscale as true ground image $I$ (pixels rescaled to $[0,1]$), and generate multiple 2D Perlin noise the same size as the image, as $C_i$, $i=1,\ldots,n$. Then the observed image $I_i$ is 
\[
I_i = C_i + (1-C_i)\circ I
\]
where pixels in $C_i$ are rescaled within $[0,1]$. Optionally one can apply any transformation  $f$ to $I$ before combining to clouds, e.g. geometric distortion to study some aspects of the methods; or generate a base $C$ and apply dynamics to $C$ for cloud time series mimicking clouds movement. We leave these for future work. By varying the parameters in Perlin noise generator, we can control the density of the generated clouds, lightly spread or heavily cover. We also apply some image correction,  e.g. Gamma correction and histogram equalisation, totally optional, to enhance the similarity to real clouds and haze. 

As the ground truth is readily accessible, we can apply any suitable quantitative evaluation to the cloud removal methods for detailed study. Given the main focus is the fidelity of the recovered image, we use the following to quantify the goodness of recovery 
\be\label{e:goodnessofrecovery}
r = \frac{\|\hat I_i - I\|_F}{\|I\|_F}
\ee
where $\hat I_i$ is the recovered image from any method. The quantity $r$ defined in \eqref{e:goodnessofrecovery} is the normalised distance metric, which is not meant to be the best. Other sophisticated measures could be applied certainly. However, \eqref{e:goodnessofrecovery} is sufficient by virtual of equivalency of norms \cite[Ch6.6]{azuhubiFunctionalAnalysis2010}, although in modelling process, different norms affect model behaviours vastly. 

\subsection{Performance evaluation on simulations on single image}
Thanks to the above semi-realistic simulation, we can now investigate another important aspect, that is the regularisation parameters used in the models. Using the goodness of recovery $r$, we can determine the best values from large scale randomised trials. Meanwhile we can also verify the necessity of the MC step. 

Let us first visually check the outcomes of different methods on one set of simulated images. The true image is from Inria dataset named {\it tyrol-w1} from Lienz in Austrian Tyrol resized to $1024\times 1024$ ($d=2^{20}$). It is a mixture of urbane and nature scene with some high intensity areas such as roads and roof tops shown in Fig. \ref{fig:simimages}. We simulate 7 thin cloud covers. One simulated image and the cloud layer are also shown in Fig. \ref{fig:simimages}. 
\begin{figure}[htbp]
	\begin{center}
		\includegraphics[width=0.8\linewidth]{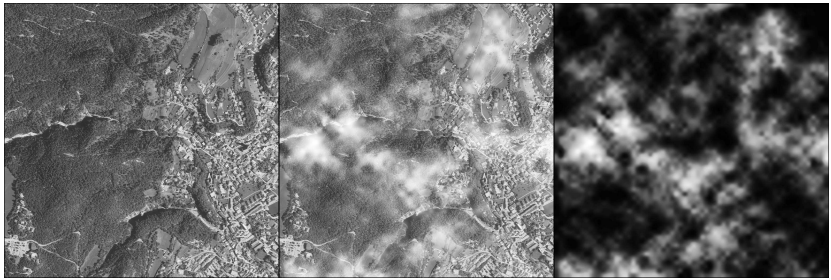}
		\caption{ Simulated image. Left to right: true clear image, one of the simulated image, its cloud cover.}
		\label{fig:simimages}
	\end{center}
\end{figure}

The clouds look very nature. Note that the cloud cover image appears to be sparse as large dark areas exist as shown in the histograms in Fig. \ref{fig:clouddetails} top panel where the right one is showing details in the range of $[0,0.2]$. However, they are not exactly zero and correspond to thin haze. If one thresholds them to zero, the cloud cover then becomes very artificial visually. The bottom panels in Fig. \ref{fig:clouddetails} show thresholding results, by 0.1 and 0.2 respectively from left to right. The visible boundaries of clouds are unpleasant and against the intuition due to the lack of the critical smoothness commonly present in natural images with clouds. This also shows the tremendous difficult to manually separate clouds in real images. 

\begin{figure}[htbp]
	\centering       
	\includegraphics[width=0.45\linewidth]{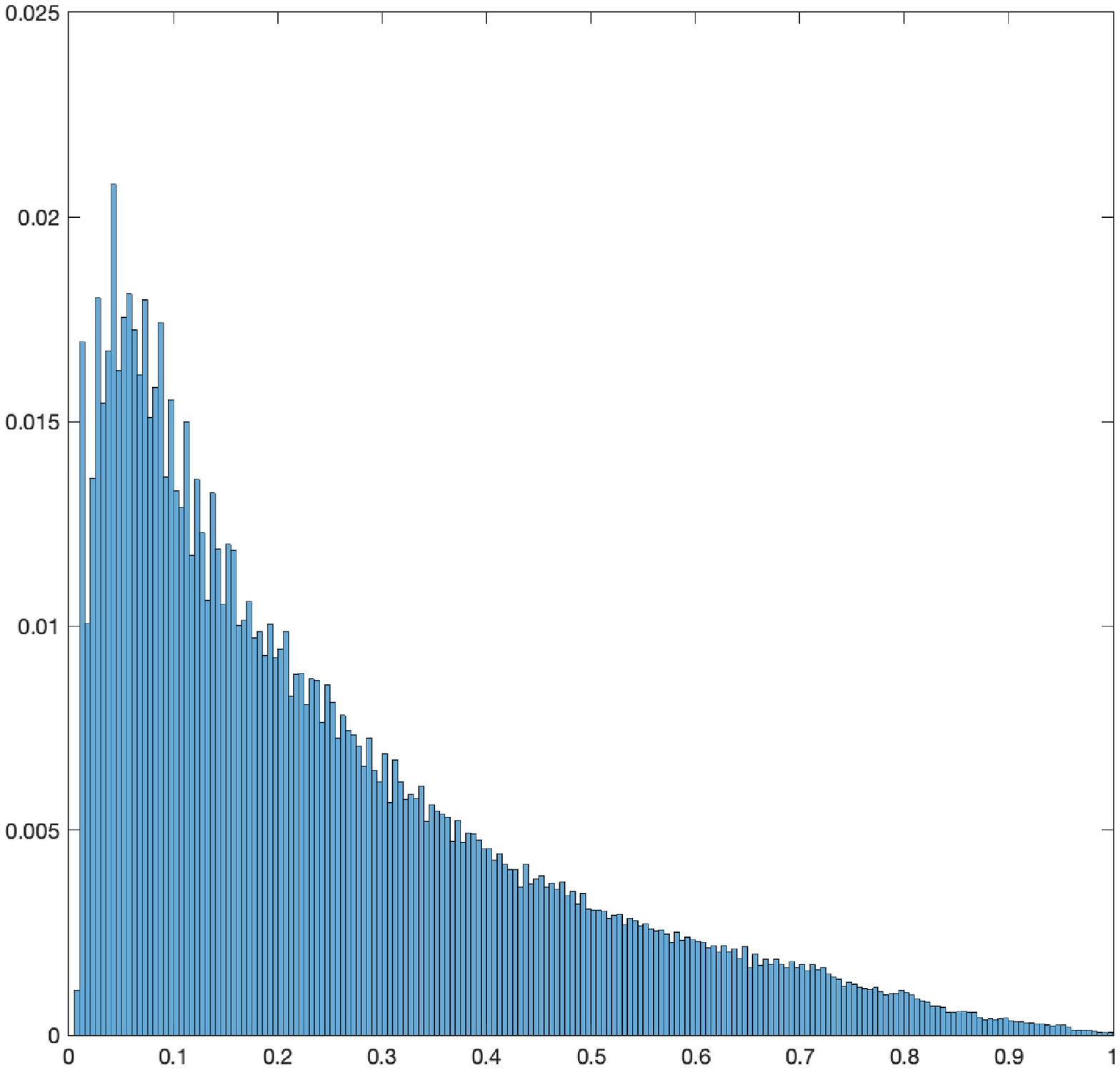}
	\includegraphics[width=0.45\linewidth]{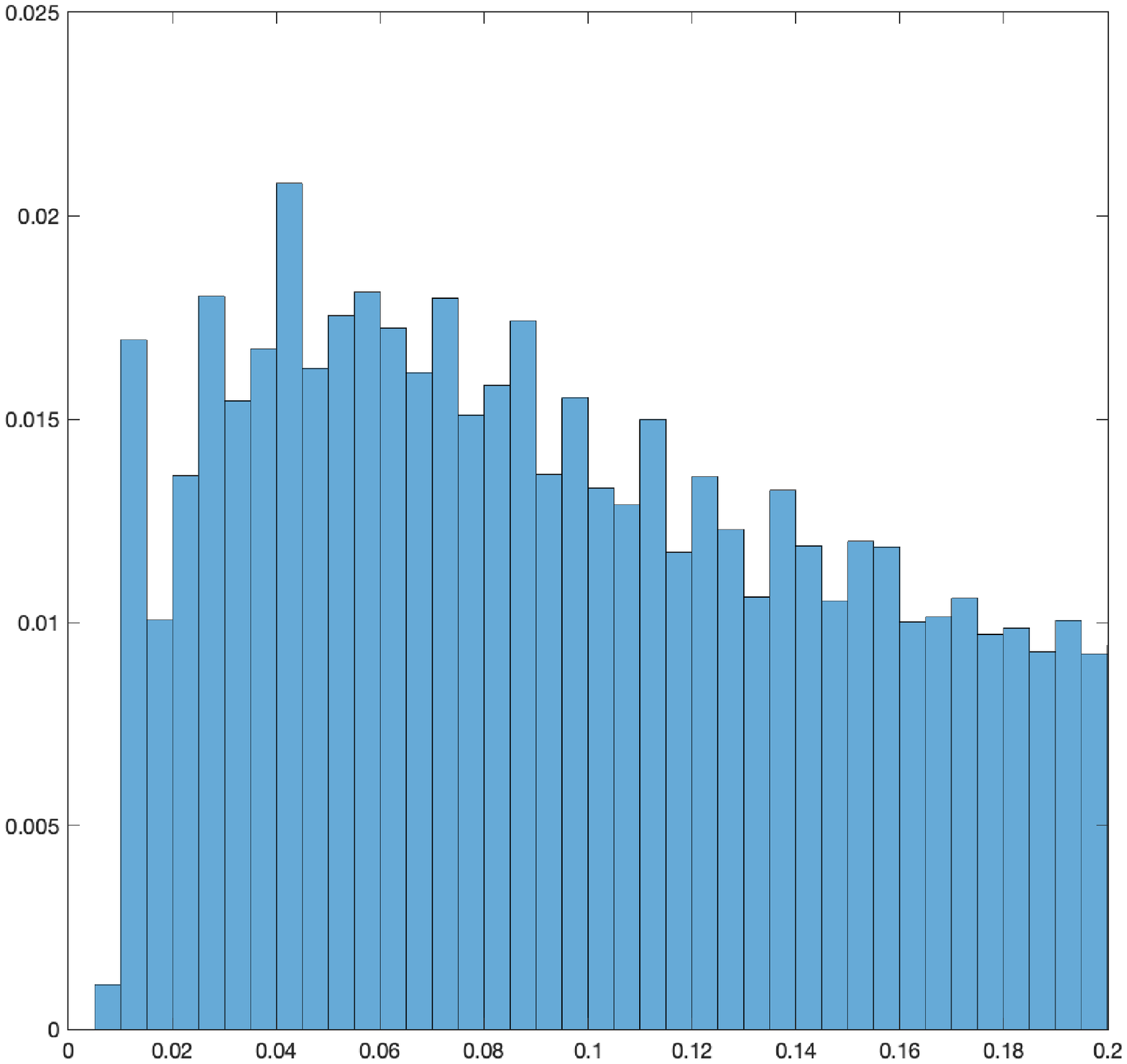}\\\ 
	\includegraphics[width=0.45\linewidth]{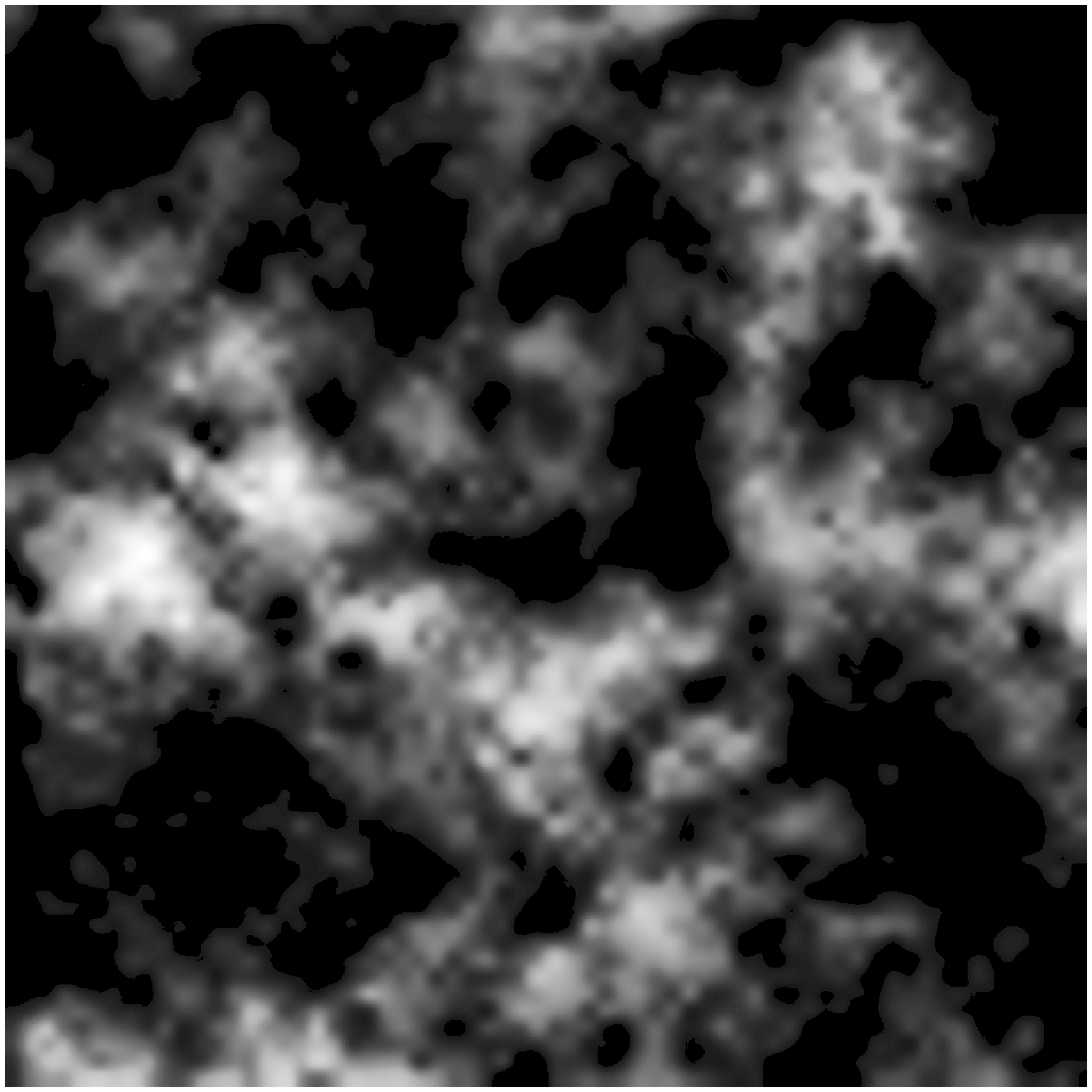}
	\includegraphics[width=0.45\linewidth]{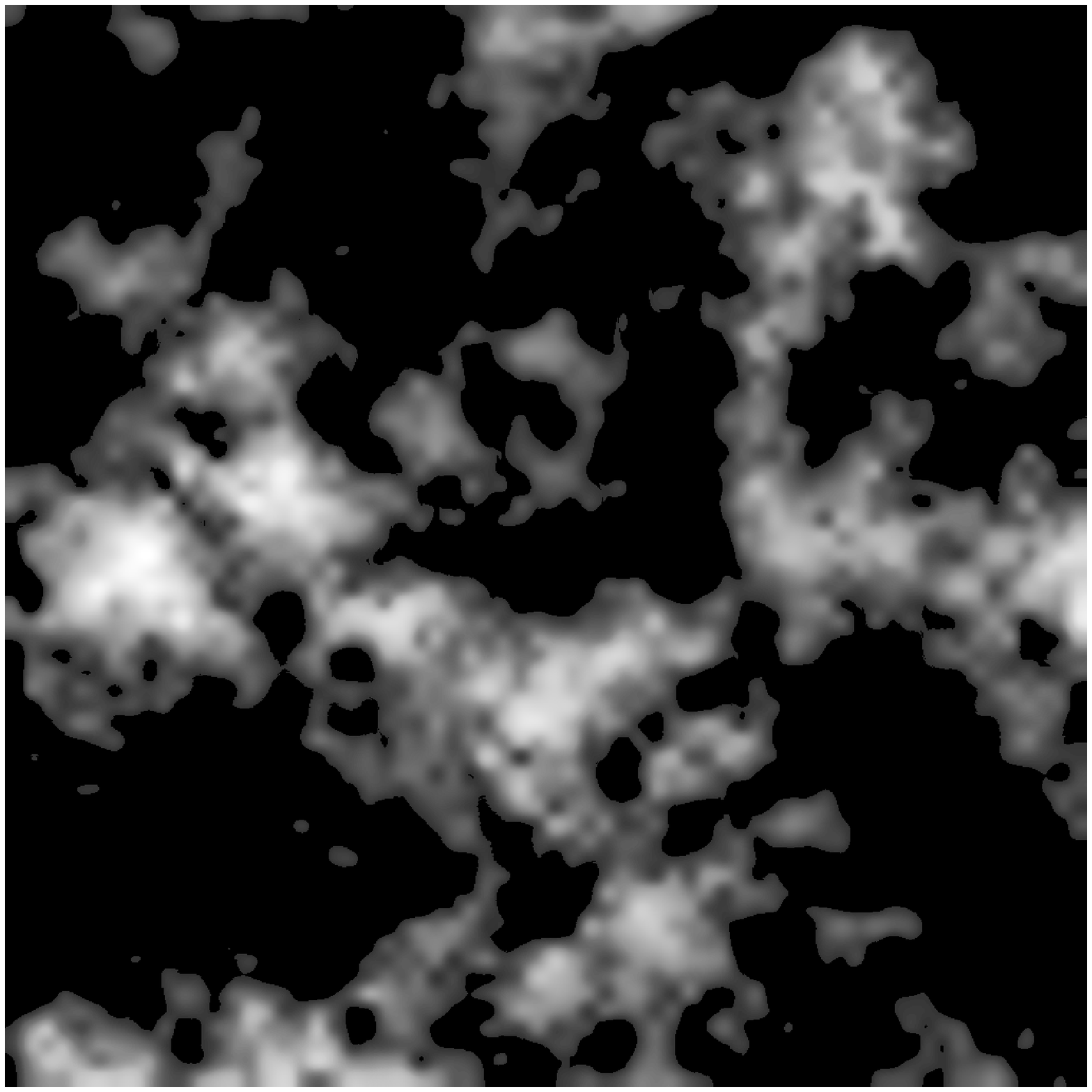}
	\caption{Details of a simulated cloud cover.}
	\label{fig:clouddetails}
\end{figure}

Fig. \ref{fig:recoveredimages} shows the recovered images by different methods obtained with the setting of the regularisation parameters as $\lambda=\frac1{\sqrt{d}}=\frac1{1024}$ and $\beta=1$ in aATM. Simple visual checking tells us that aATM and RPCA are better than ATM as ATM results (with and without MC)  contain fair amount of cloud pixels. This may be straightforward. However, it is not clear which one is the best. It appears that aATM is slightly better for less ``washed away'' areas. It is also impossible to identify the effect of MC. These indicate the limit of visual examination. Nonetheless, the $r$ values of these methods are 0.1758 ,0.3195, 0.1754 in the order of aATM, ATM, RPCA with MC, and {\bf 0.1678}, 0.3373, 0.1681 without MC. Now it is clear that aATM without MC is the best and {\it MC does not do anything useful to enhance the results}. 

\begin{figure}[htbp]
	\begin{center}
		\includegraphics[width=0.99\linewidth]{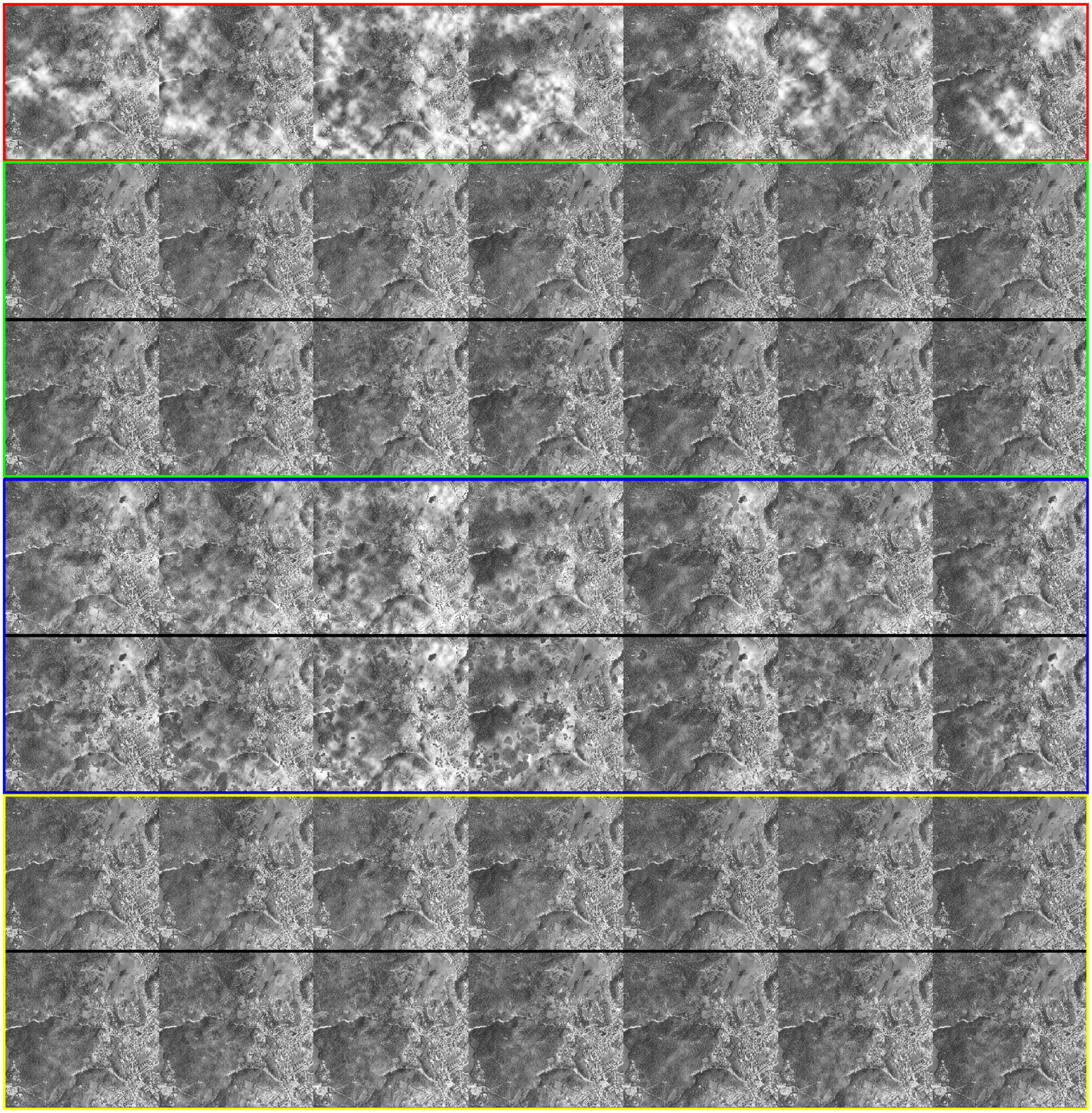}
		\caption{ Recovered images. top row: simulated images with cloud cover; 2nd-3rd row (in green box): aATM results with/without MC; 4th-5th row (in blue box): ATM results with/without MC; 6th-7th row (in yellow box): RPCA results with/without MC.}
		\label{fig:recoveredimages}
	\end{center}
\end{figure}

\begin{figure}[htbp]
	\begin{center}
		\includegraphics[width=0.99\linewidth]{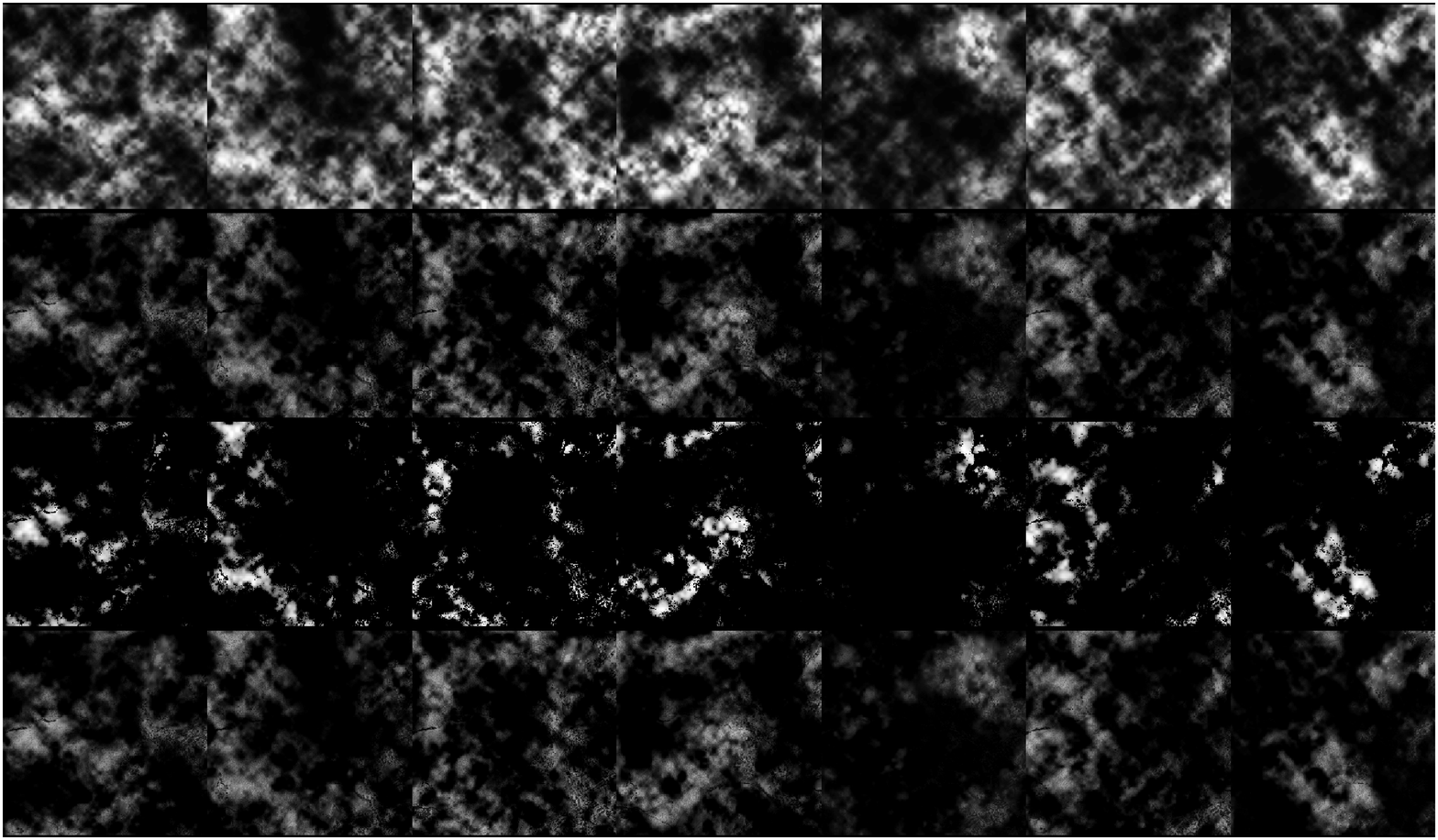}
		\caption{Cloud covers estimated by various methods. From top: known cloud covers, aATM results, ATM results, RPCA results.}
		\label{fig:cloudcover}
	\end{center}
\end{figure}

Fig. \ref{fig:cloudcover} shows the cloud covers detected by these methods. The clouds separated by ATM are in better contrast, i.e. very bright and very dark although it appears very conservative, that is visually sparser than others.  In contrast, aATM and RPCA seem to have more cloud pixels identified. Again, it is impossible to tell the difference between aATM clouds and RPCA clouds by visual examination. Note that for aATM the cloud is the summation of $C$ and $N$. 

One major benefit of simulation is to validate the sensitivity of the regularisation parameter, mainly $\lambda$ in the models. We ran large scale simulation with $n=7$ and 15, using the same true image. We tested 51 values of $\lambda$ equally spaced in log scale with the recommend value, $\frac1{\sqrt{d}}$, in the middle, i.e. from 9.7656e-05 to 0.0098, and for each value of $\lambda$, we ran 50 randomised trials. The results are collected in Fig. \ref{fig:sim7} and \ref{fig:sim15}, where each data point is the mean and $\pm1$ standard deviation of the $r$ values across all trials for a given $\lambda$ value. 

Many things can be read out from the plots. The first is that ATM is not as good as competitors for small $n$, e.g. $n=7$, regardless the choice of the $\lambda$ values. 
However, it begins to gain advantage when $n$ is larger. This will be investigated later. The second is that $\lambda$ has roughly 3 zones: 1) failure zone, where the sparsity is too weak and all methods fail with no recovered images; 2) clamping zone, where the sparsity is overwhelming such that sparse component is wiped out and all methods lose the capacity to identify clouds; 3) Goldilock zone, where the algorithms work reasonably well ($r\le0.2$ for $\lambda\in[0.0007,0.0012]$), including their bests. Of course, these zones have different boundaries for different methods, and their $r$ values inside these zones have different shapes. For example, RPCA seems to have rather flat $r$ values in its Goldilock zone meaning that its performance varies just a little bit if $\lambda$ is from that zone. There exist a value for $\lambda$ which is better than the default recommended value. This holds for all methods, interestingly with different margin of being true. For example, for RPCA, the margin is smaller, that is the optimal value of $\lambda$ brings 17.23\% reduction of $r$ value on average in $n=7$ case, while that is 42.11\% for aATM. Similar observation for $n=15$. When all methods take the default value of $\lambda$, aATM without MC works the best on average, which is 22.84\% better than RPCA in expectation sense. The overall best performance of aATM against that of RPCA is 43.06\% reduction in $r$ value, down from 0.1625 to 0.0941, which is very significant. This is verified by a one-side t-test with null hypothesis of no $r$ values reduction performed on the trials with the optimal and default $\lambda$ values where significance level $\alpha=10^{-5}$. The resulting p-value for null hypothesis is extremely low $1.7557\times 10^{-29}$ strongly supporting the alternative hypothesis that the reduction is quite significant. A very interesting observation is that ATM without MC comes to the second when $n=15$ in terms of the overall best performance, better than RPCA. Fig. \ref{fig:ATMvsRPCAsim15} reveals the details of the $r$ values of both methods in the trials when holding $\lambda$ value constant, $\lambda=4.6741$e$-04$, the optimal value for both methods. The $r$ values of each method vary during the trials due to the randomness of the simulation. RPCA has higher values of $r$ almost constantly with greater variation than ATM. There is no doubt that ATM outperforms RPCA when $\lambda$ is optimal. 
The third is that MC does not bring much improvement even acts adversely when $\lambda$ is in the Goldilock zone. This claim is strongly supported by statistical evidence. Table \ref{table:testMC} shows the one-side t-tests results performed on the trials of various methods with optimal $\lambda$ values for both $n=7$ and $n=15$ cases. The null hypothesis is that MC brings $r$ value reduction on average, i.e. the mean of $r_{MC}-r_{\overset{\sim}{MC}}$ is no greater than 0. $r_{MC}$ and $r_{\overset{\sim}{MC}}$ are the $r$ values of a method with and without MC respectively. The significance level $\alpha$ is set as low as $10^{-5}$.  The p-values are extremely low suggesting that the null hypothesis should be rejected almost surely. The only exception is ATM when $n=7$, which favours the MC to further improve its performance.  
So clearly the recommendation is to omit MC step in cloud removal in these methods, which is extra computation with little benefit. However, we need to point out here though that there are regularisation parameters as well in MC, for which we took the default/recommended values, see previous sections for detail. 

\begin{table}                                     
	\centering                                        
	\begin{tabular}{c|cp{2.5cm}|cp{2.5cm}}   
		\hline
		\multirow{2}{*}{Methods} & \multicolumn{2}{c|}{$n=7$} &  \multicolumn{2}{c}{$n=15$} \\
		\cline{2-5}
		& p-value & Confidence interval & p-value & Confidence interval\\  \hline           
		aATM  & 1.4973e-58 &  $[0.0430, \infty)$  & 1.0251e-74 & $[0.0629, \infty)$ \\
		ATM & 1.0000  & $[-0.0234,  \infty)$   & 5.8915e-31& $[0.0155, \infty)$ \\
		RPCA & 1.7878e-17 & $[0.0039, \infty)$  & 7.7511e-44 & $[0.0126, \infty)$ \\ \hline
	\end{tabular}                                     
	\caption{Null hypothesis $H_0$ is $\overline{r_{MC}-r_{\overset{\sim}{MC}}} \le 0$. Significance level $\alpha=10^{-5}$ in t-tests.}                               
	\label{table:testMC}                        
\end{table} 

\begin{figure}[htbp]
	\begin{center}
		\includegraphics[width=0.45\linewidth]{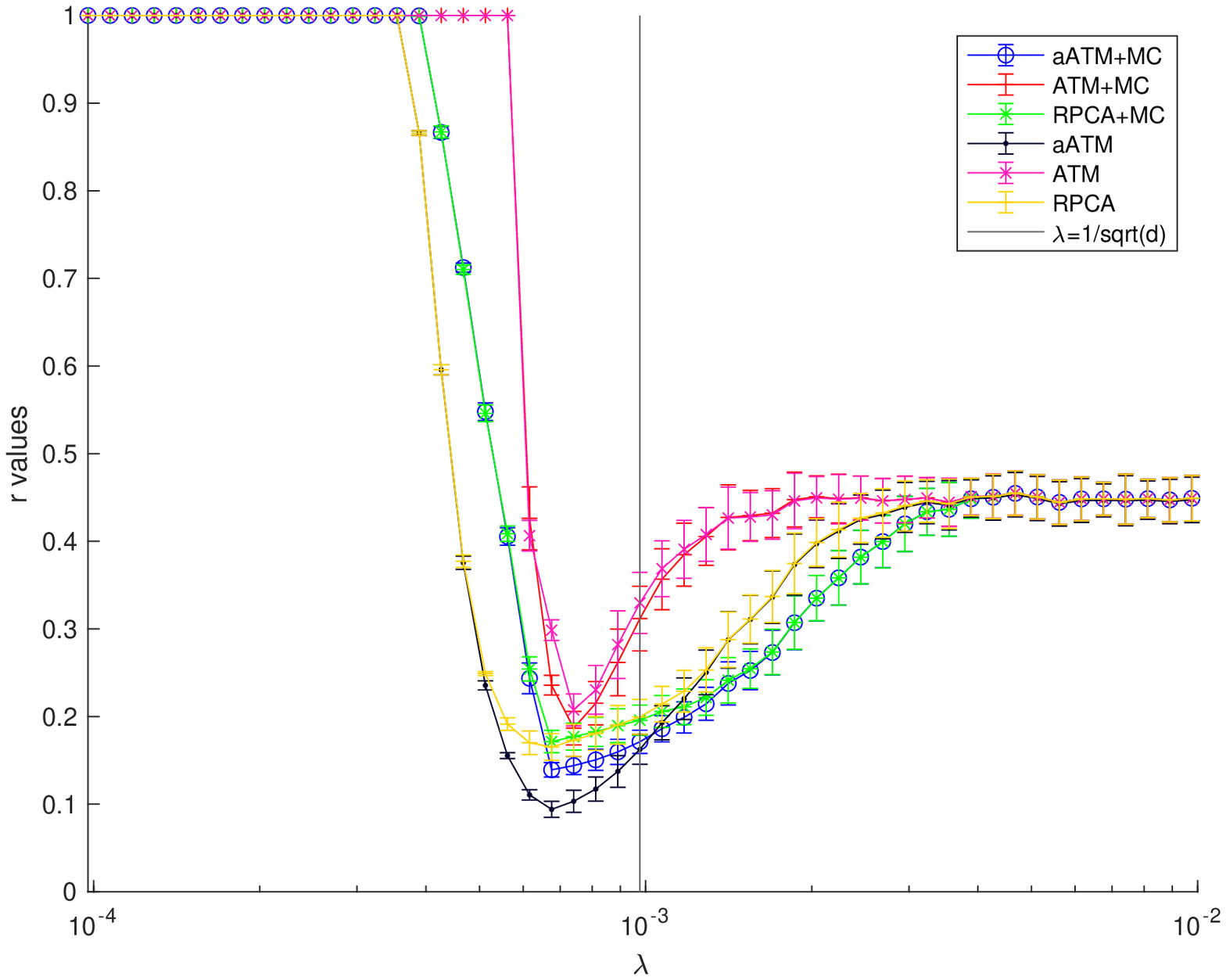}
		\includegraphics[width=0.45\linewidth]{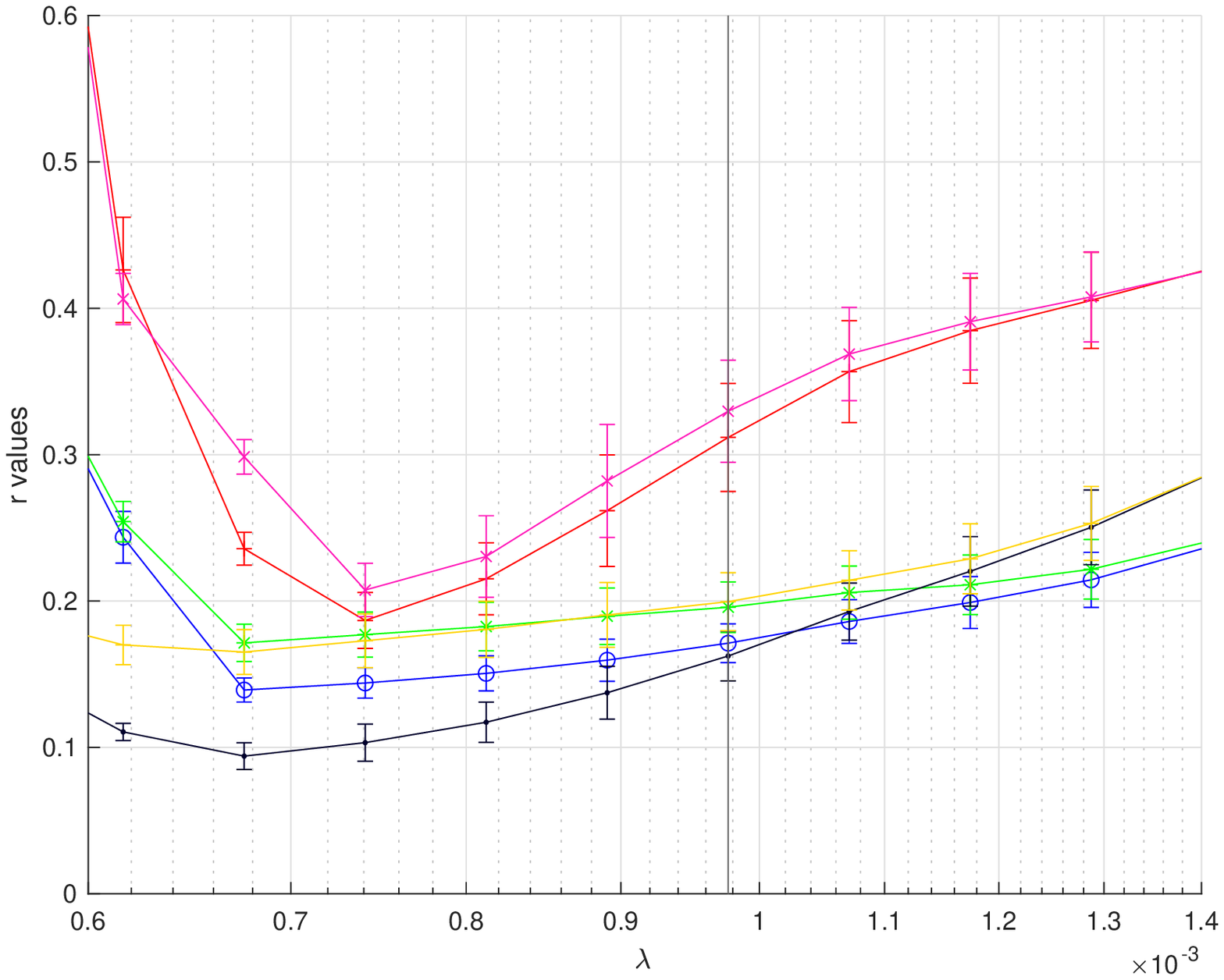}
		\caption{Randomised trials ($n=7$). $\lambda$ in log scale. }
		\label{fig:sim7}
	\end{center}
\end{figure}

\begin{figure}[htbp]
	\begin{center}
		\includegraphics[width=0.45\linewidth]{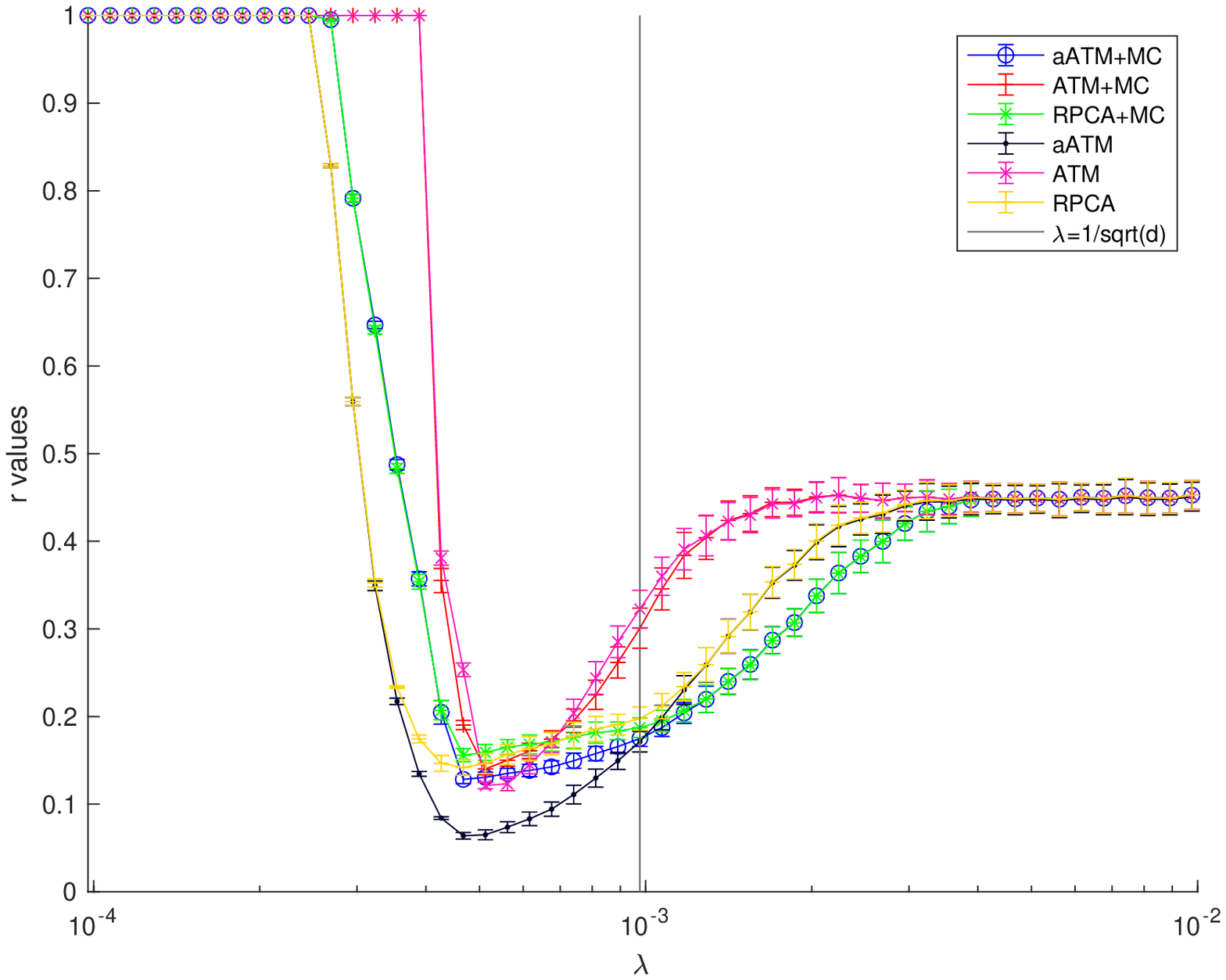}
		\includegraphics[width=0.45\linewidth]{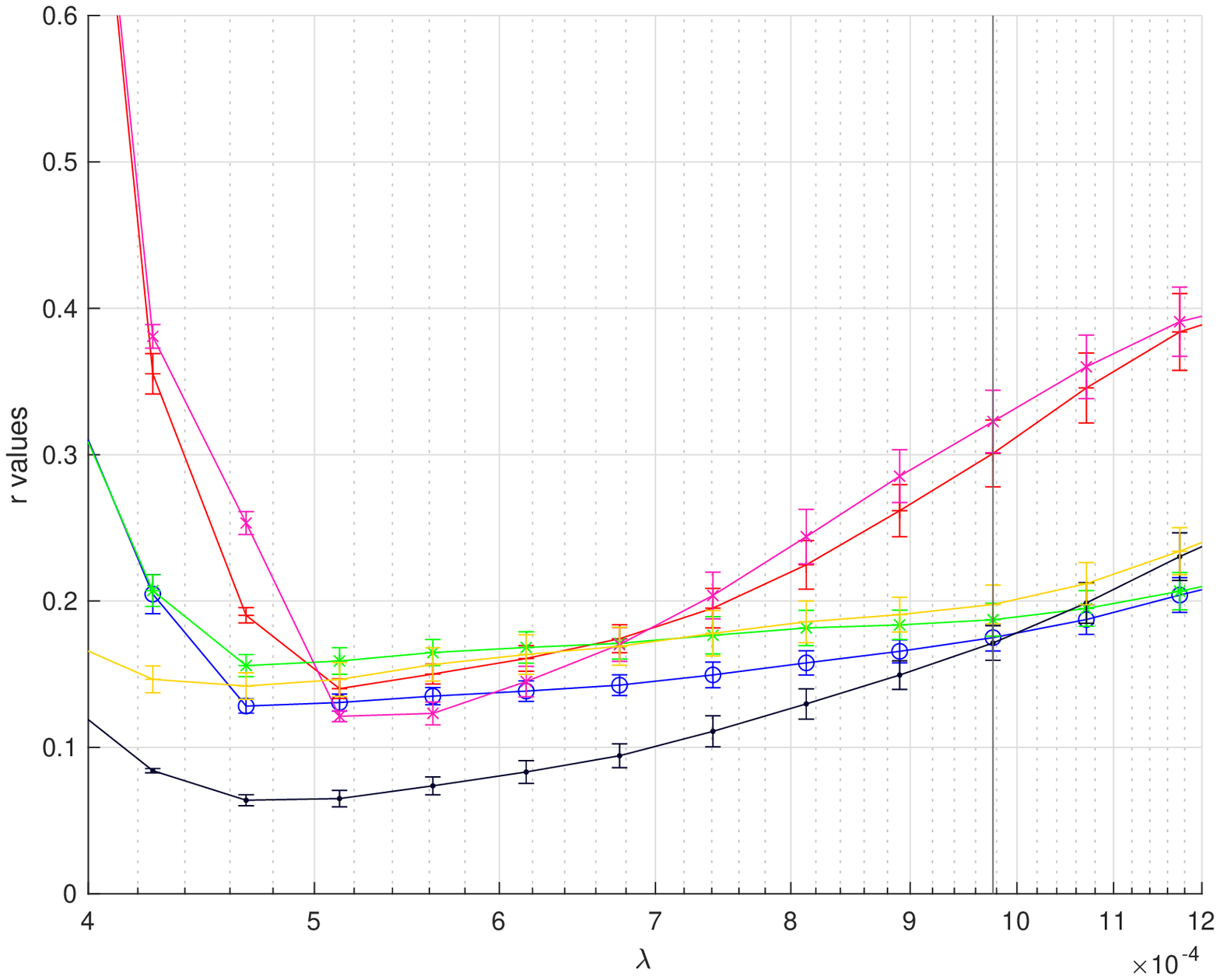}
		\caption{Randomised trials ($n=15$). $\lambda$ in log scale. }
		\label{fig:sim15}
	\end{center}
\end{figure}

\begin{figure}[htbp]
	\begin{center}
		\includegraphics[width=0.45\linewidth]{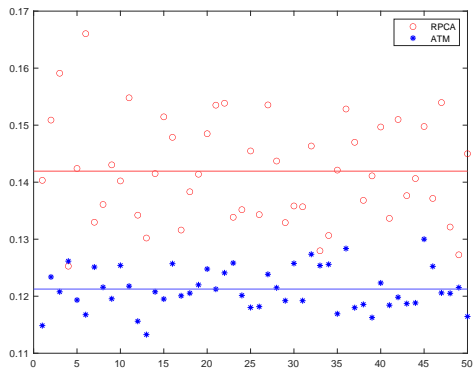}
		\caption{$r$ values of ATM and RPCA in randomised trials ($n=15$) with their best  $\lambda$ values ($\lambda=4.6741$e$-04$ for both). X-axis is the repeat indices from 1 to 50. The coloured horizontal lines are the means of corresponding $r$ values.}
		\label{fig:ATMvsRPCAsim15}
	\end{center}
\end{figure}

To see the comparison more clearly, we present the mean and standard deviation values of the results in some $\lambda$ range (in the Goldilock zone) into Table \ref{table:mean} and \ref{table:std} for clarity. The column in the middle of the tables with double column indicates the values when $\lambda$ equal to the default value. The column-wise best (minimum among all methods) is highlighted by italic font and overall best is highlighted by bold font. They show clearly that aATM is the best in terms of both expected $r$ value and stability reflected by  smaller standard deviations. 

\begin{table*}                                                                                                  
	\centering                                                                                \resizebox{\textwidth}{!}{                    
	\begin{tabular}{c|ccccc||c||ccc}      
		\hline                                                                              
		$\lambda\ (\times 10^{-3})$ & 0.6162 & 0.6756 & 0.7408 & 0.8123 & 0.8906 & 0.9766 & 1.0708 & 1.1741 & 1.2874 \\\hline
		aATM+MC & 0.2436 & 0.1392 & 0.1440 & 0.1506 & 0.1596 & 0.1712 & {\it 0.1860} & {\it 0.1990} & {\it 0.2145} \\                    
		ATM+MC & 0.4262 & 0.2358 & 0.1867 & 0.2152 & 0.2618 & 0.3118 & 0.3567 & 0.3847 & 0.4054 \\                     
		RPCA+MC & 0.2543 & 0.1714 & 0.1771 & 0.1826 & 0.1896 & 0.1958 & 0.2058 & 0.2111 & 0.2217 \\                    
		aATM  & {\it 0.1106} & {\bf 0.0941} & {\it 0.1033} & {\it 0.1172} & {\it 0.1374} & {\it 0.1625} & 0.1928 & 0.2203 & 0.2504 \\                      
		ATM & 0.4063 & 0.2985 & 0.2075 & 0.2304 & 0.2820 & 0.3297 & 0.3687 & 0.3908 & 0.4078 \\                        
		RPCA & 0.1700 & 0.1652 & 0.1729 & 0.1806 & 0.1905 & 0.1996 & 0.2142 & 0.2288 & 0.2531 \\                       
		\hline
	\end{tabular}}                                                                                                
	\caption{Mean $r$ values of various methods from all trials for some $\lambda$ values. }                                                                                                   
	\label{table:mean}                                                                                     
\end{table*}

\begin{table*}                                                                                                  
	\centering                                                                                \resizebox{\textwidth}{!}{                   
	\begin{tabular}{c|ccccc||c||ccc}   
		\hline                                                                                 
		$\lambda\ (\times 10^{-3})$ & 0.6162 & 0.6756 & 0.7408 & 0.8123 & 0.8906 & 0.9766 & 1.0708 & 1.1741 & 1.2874 \\\hline
		aATM+MC & 0.0176 & {\it 0.0082} & {\it 0.0104} & {\it 0.0120} & {\it 0.0144} & 0.0132 & {\it 0.0149} & {\it 0.0177} & {\it 0.0188} \\                    
		ATM+MC & 0.0360 & 0.0112 & 0.0191 & 0.0246 & 0.0381 & 0.0369 & 0.0348 & 0.0360 & 0.0329 \\                     
		RPCA+MC & 0.0138 & 0.0127 & 0.0154 & 0.0166 & 0.0193 & 0.0173 & 0.0182 & 0.0204 & 0.0203 \\                    
		aATM  & {\bf 0.0059}& 0.0091 & 0.0126 & 0.0138 & 0.0181 & 0.0171 & 0.0195 & 0.0237 & 0.0255 \\                      
		ATM & 0.0175 & 0.0118 & 0.0182 & 0.0278 & 0.0386 & 0.0348 & 0.0318 & 0.0330 & 0.0307 \\                        
		RPCA & 0.0134 & 0.0153 & 0.0183 & 0.0192 & 0.0221 & 0.0198 & 0.0203 & 0.0239 & 0.0253 \\  \hline                     
	\end{tabular}}                                                                                                  
	\caption{Standard deviations of $r$ values of various methods from all trials for some $\lambda$ values.}                                                                                                   
	\label{table:std}                                                                                     
\end{table*}

\subsection{Computation costs comparison on simulations using single fixed image}
We report the time for computation. Fig. \ref{fig:sim7time} shows the time consumed by various methods, with and without MC. Similar to previous plots, the data points in the plot are the means and $\pm1$ standard deviation of the times (in seconds) across all trials for a given $\lambda$ value. Apparently they vary across simulations.

Quite obviously here MC is extra work. Given no extra benefit, the computation for MC should be saved. ATM is pretty difficult to solve indeed, reflected by the skyrocketed computational time compared with those from others. There are double optimisation loops inside its solver. Interestingly, when $\lambda$ is correct, ATM takes the most of time to compute {\it on average}. When $\lambda$ is growing from the failure zone to the Goldilock zone, a huge jump of needed computation can be observed, which is statistically significant. As ATM's performance turns very sharply along $\lambda$ values, its computation cost varies accordingly, peaking at where ATM works the best and jumping down quickly. This is a very interesting observation that may lead to a way of selection of its regularisation parameter as well as a hypothesis of required computational cost vs $\lambda$ value. Along with the well known regularisation path in sparse models \cite{HastieRossetTibshiraniZhu2004,FriedmanHastieTibshirani2010a,TibshiraniTaylor2011}, this may be a useful route leading to optimal regularisation selection in future. This is never possible previously without simulation. In general, aATM is more expensive to compute than RPCA because of the extra block of variables $N$, doubling the cost almost for all $\lambda$ values. However, the base is quite small. when $\lambda$ is in the Goldilock zone, aATM is doubling RPCA from about 6 seconds to 10 seconds. Therefore it is not dramatic. 

\begin{figure}[htbp]
	\begin{center}
		\includegraphics[width=0.45\linewidth]{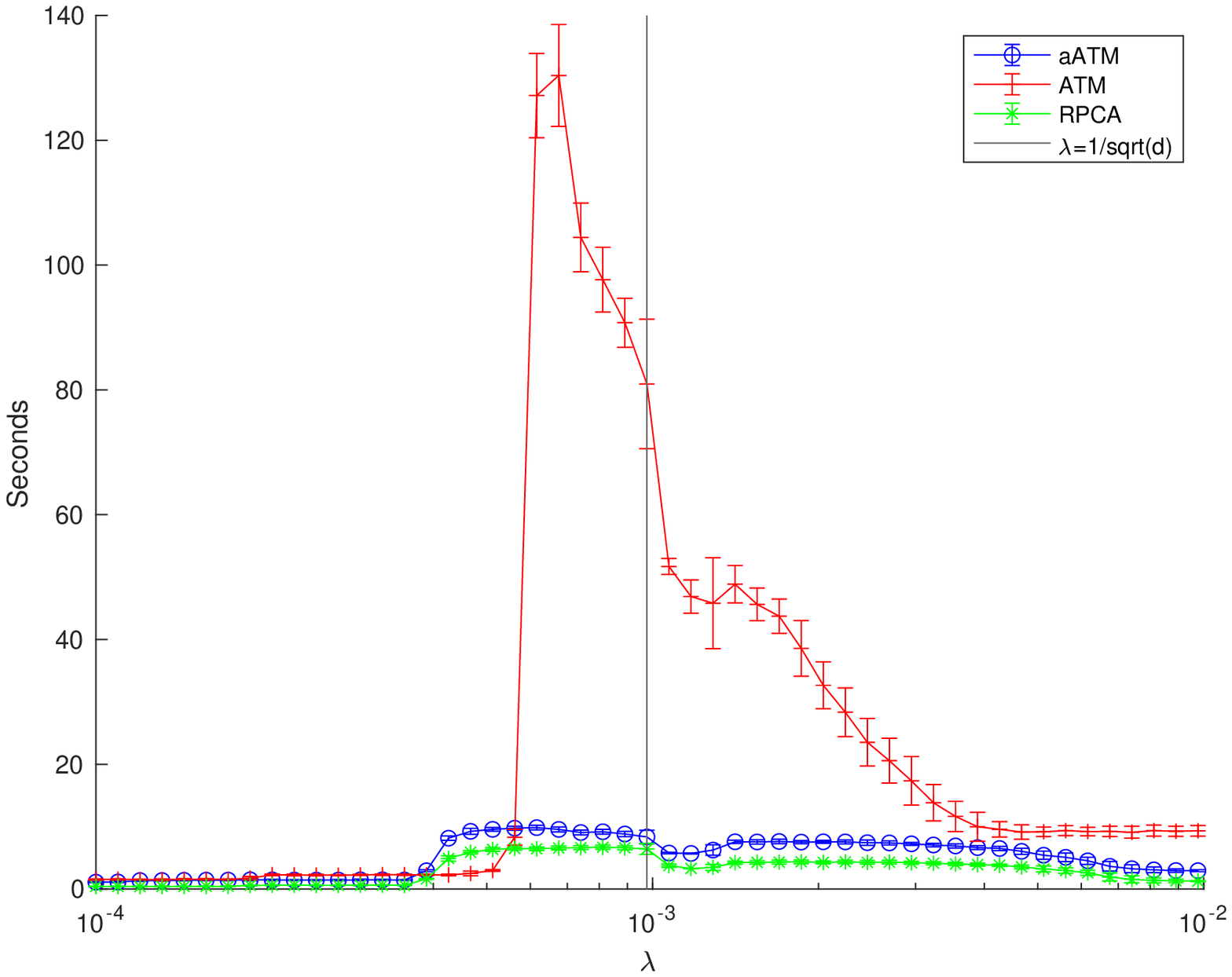}
		\includegraphics[width=0.45\linewidth]{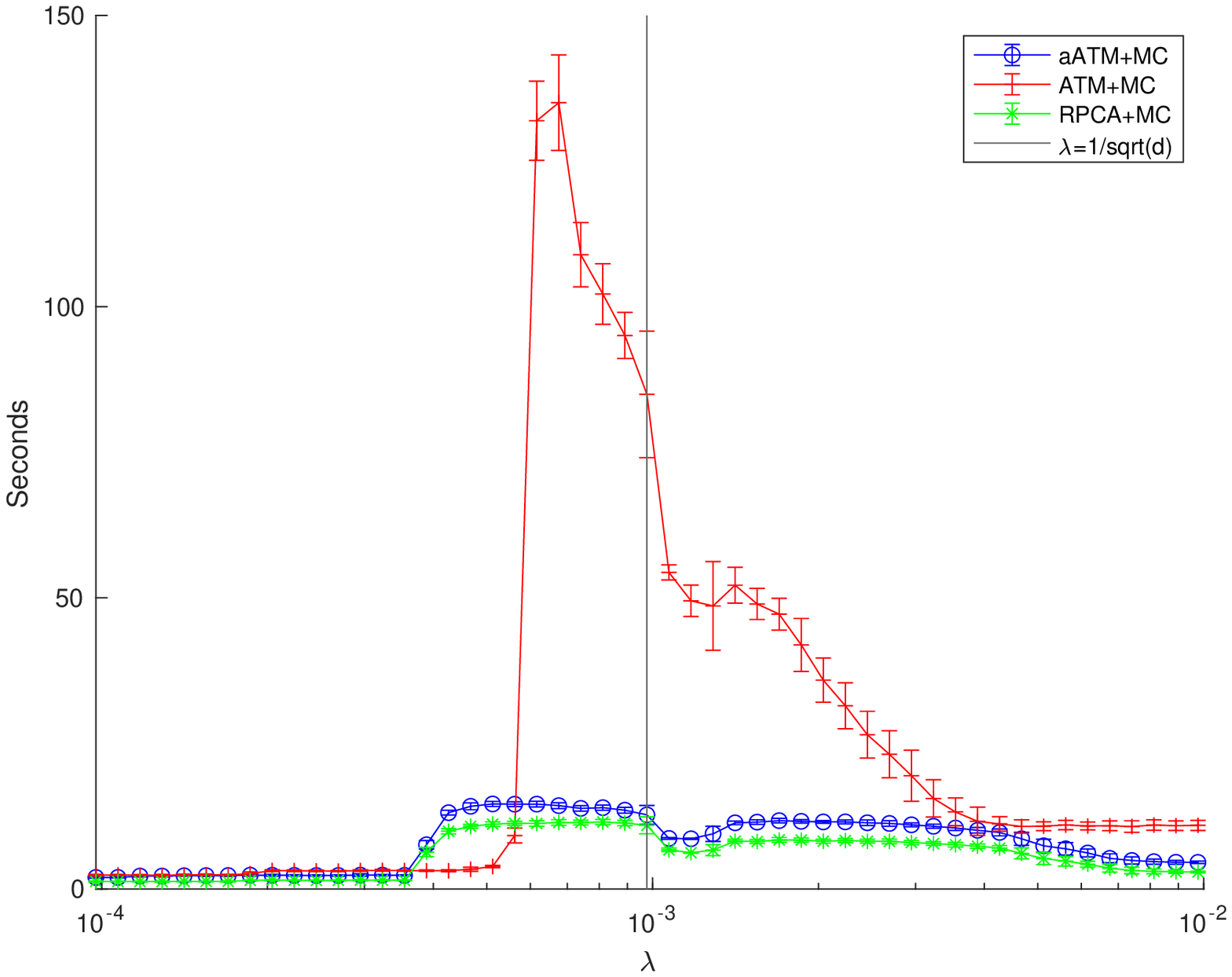}
		\caption{Time consumed in randomised trials ($n=7$). $\lambda$ in log scale. }
		\label{fig:sim7time}
	\end{center}
\end{figure}

Again, we present the mean and standard deviation values of the results in some $\lambda$ range (in the Goldilock zone) into Table \ref{table:meantime} and \ref{table:stdtime} for clarity. The column in the middle of the tables with double column indicates the values when $\lambda$ equal to the default value. The column-wise best (minimum time among all methods) is highlighted by italic font and overall best is highlighted by bold font. They show clearly that RPCA is the fastest and aATM is about 50\% more expensive to run at this range of $\lambda$ values. Considering its superiority in recovery performance, this cost is absolutely worthwhile.  

\begin{table*}                                                                                                  
	\centering                                                                                \resizebox{\textwidth}{!}{                     
	\begin{tabular}{c|ccccc||c||ccc}           
		\hline                                                                       
		$\lambda\ (\times 10^{-3})$ & 0.6162 & 0.6756 & 0.7408 & 0.8123 & 0.8906 & 0.9766 & 1.0708 & 1.1741 & 1.2874 \\\hline
		aATM  & 9.84 & 9.56 & 9.06 & 9.18 & 8.82 & 8.35 & 5.77 & 5.71 & 6.25 \\                      
		ATM & 127.17 & 130.38 & 104.43 & 97.66 & 90.77 & 80.95 & 51.69 & 46.87 & 45.81 \\            
		RPCA & {\it 6.45} & {\it 6.58} & {\it 6.64} & {\it 6.70} & {\it 6.61} & {\it 6.42} & {\it 3.75} & {\bf 3.30} & {\it 3.47} \\                       
		aATM+MC & 14.59 & 14.32 & 13.85 & 13.97 & 13.54 & 12.73 & 8.69 & 8.63 & 9.49 \\              
		ATM+MC & 131.92 & 135.02 & 108.91 & 102.17 & 95.06 & 84.95 & 54.37 & 49.49 & 48.61 \\        
		RPCA+MC & 11.26 & 11.43 & 11.42 & 11.51 & 11.34 & 10.94 & 6.72 & 6.24 & 6.69 \\        
		\hline        
	\end{tabular}}                                                                                                  
	\caption{Means of time consumed by various methods from all trials for some $\lambda$ values.}                                                                                                   
	\label{table:meantime}                                                                                     
\end{table*}  

\begin{table*}                                                                                                  
	\centering                                                                                \resizebox{\textwidth}{!}{                      
	\begin{tabular}{c|ccccc||c||ccc}           
		\hline                                                                       
		$\lambda\ (\times 10^{-3})$ & 0.6162 & 0.6756 & 0.7408 & 0.8123 & 0.8906 & 0.9766 & 1.0708 & 1.1741 & 1.2874 \\\hline
		aATM  & 0.36 & 0.40 & 0.39 & 0.38 & 0.42 & 1.07 & 0.13 & 0.07 & 0.84 \\                      
		ATM & 6.75 & 8.16 & 5.50 & 5.19 & 3.93 & 10.38 & 1.28 & 2.66 & 7.29 \\                       
		RPCA & {\it 0.30} & {\it 0.32} & {\it 0.31} & {\it 0.33} & {\it 0.31} & {\it 0.87} & 0.28 & \bf{0.04} & {\it 0.48} \\                       
		aATM+MC & 0.48 & 0.51 & 0.47 & 0.44 & 0.50 & 1.62 & 0.13 & 0.11 & 1.31 \\                    
		ATM+MC & 6.80 & 8.20 & 5.52 & 5.20 & 3.97 & 10.86 & 1.29 & 2.70 & 7.62 \\                    
		RPCA+MC & 0.48 & 0.48 & 0.41 & 0.46 & 0.46 & 1.45 & 0.30 & 0.08 & 0.94 \\      
		\hline              
	\end{tabular}}                                                                                
	\caption{Standard deviations of time consumed by various methods from all trials for some $\lambda$ values.}                                                                                 
	\label{table:stdtime}                                                                   
\end{table*}

Both aATM and RPCA exhibit the same pattern observed from ATM but less pronounced. When $\lambda$ goes form failure zone to Goldilock zone, there is time cost leap and stabilises for a while and then some up and downs. Again, the zone changing pattern of time cost is a good indicator of entering the Goldilock zone from failure zone. It is possible to exploit it for finding a better $\lambda$ value than the default one, although it is tricker than ATM where the pattern is very clear. 

\subsection{Determining the best $\lambda$ value}\label{sec:bestlamba}
What is the best value for the regularisation parameter $\lambda$? This is an inevitable and yet critical question in practice. It is almost impossible to address it without many assumptions and lengthy theoretic analysis. Please refer to Section \ref{sec:analysis} for Goldilock zone bounds for the complexity. However, thanks to simulation, we can fit the data to derive some equation for the best $\lambda$ value. Different from drilling into the computational cost pattern suggested by previous section, we look at the best $\lambda$ values of different methods by stretching $n$ from 2 to 250. The ``best'' is defined as the $\lambda$ value corresponding to the minimum average $r$ value across trials, which we denote as $\lambda^*$. Fig. \ref{fig:bestlambdavsn} shows the ratio of $\lambda^*$ of all methods to the suggested default value $\frac1{\sqrt d}$, i.e. $\sqrt d\lambda^*$. As $n$ becomes larger, $\sqrt d\lambda^*$ decreases exponentially. We turn this into almost linear by applying $\log$ twice to $n$, as shown in Fig. \ref{fig:bestlambdavsn} right panel. From this data, we fit a linear model and derive the following $\lambda^*$ estimator 
\begin{equation}\label{e:bestlambda}
\hat\lambda^* = \frac{-0.5682\log(\log(n))+1.0747}{\sqrt d}, \forall n\ge 2
\end{equation}
The red curves in Fig. \ref{fig:bestlambdavsn} are the values of $\sqrt d\hat\lambda^*$ at different scales of $n$. 

\begin{figure}[htbp]
	\begin{center}
		\includegraphics[width=0.45\linewidth]{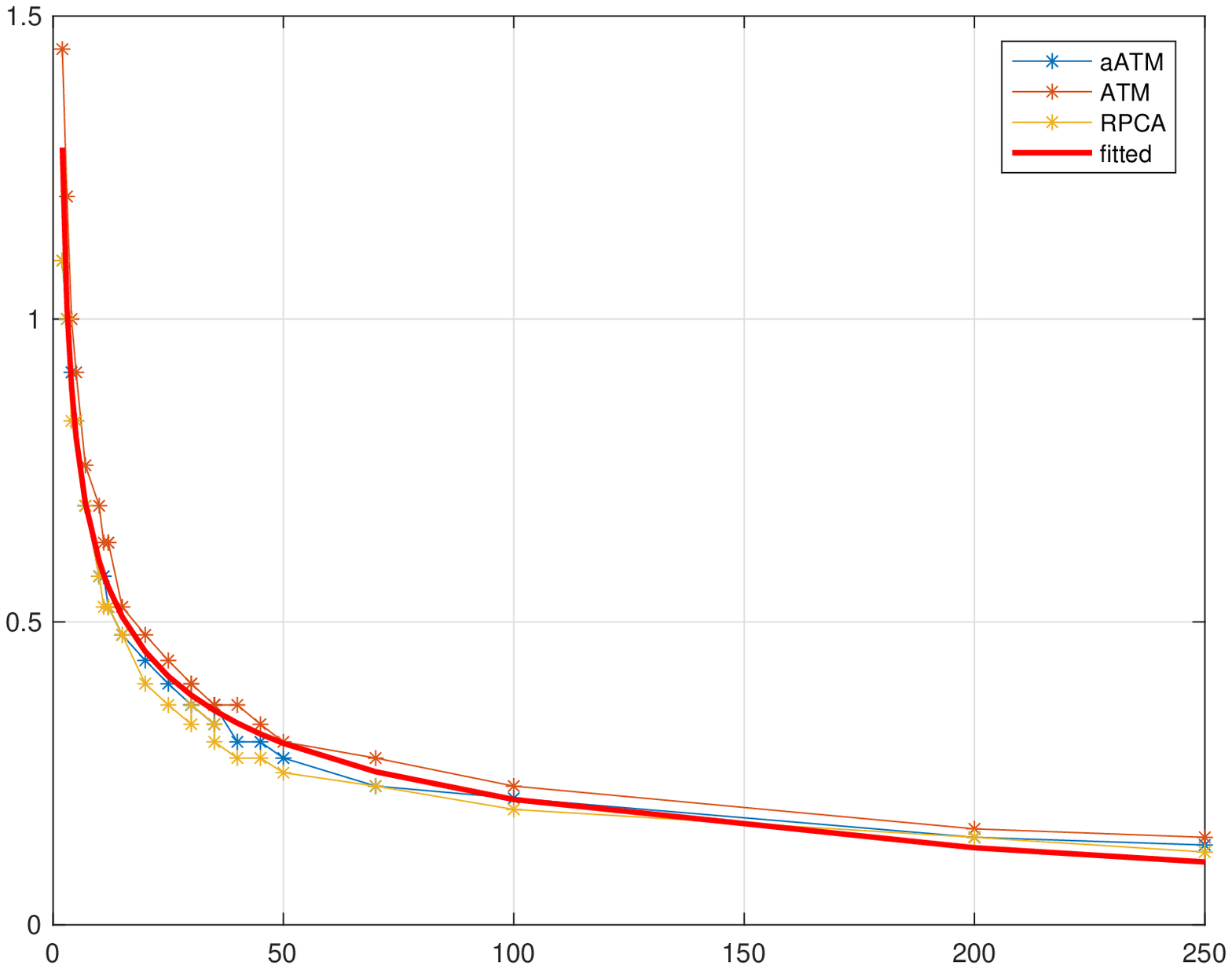}
		\includegraphics[width=0.45\linewidth]{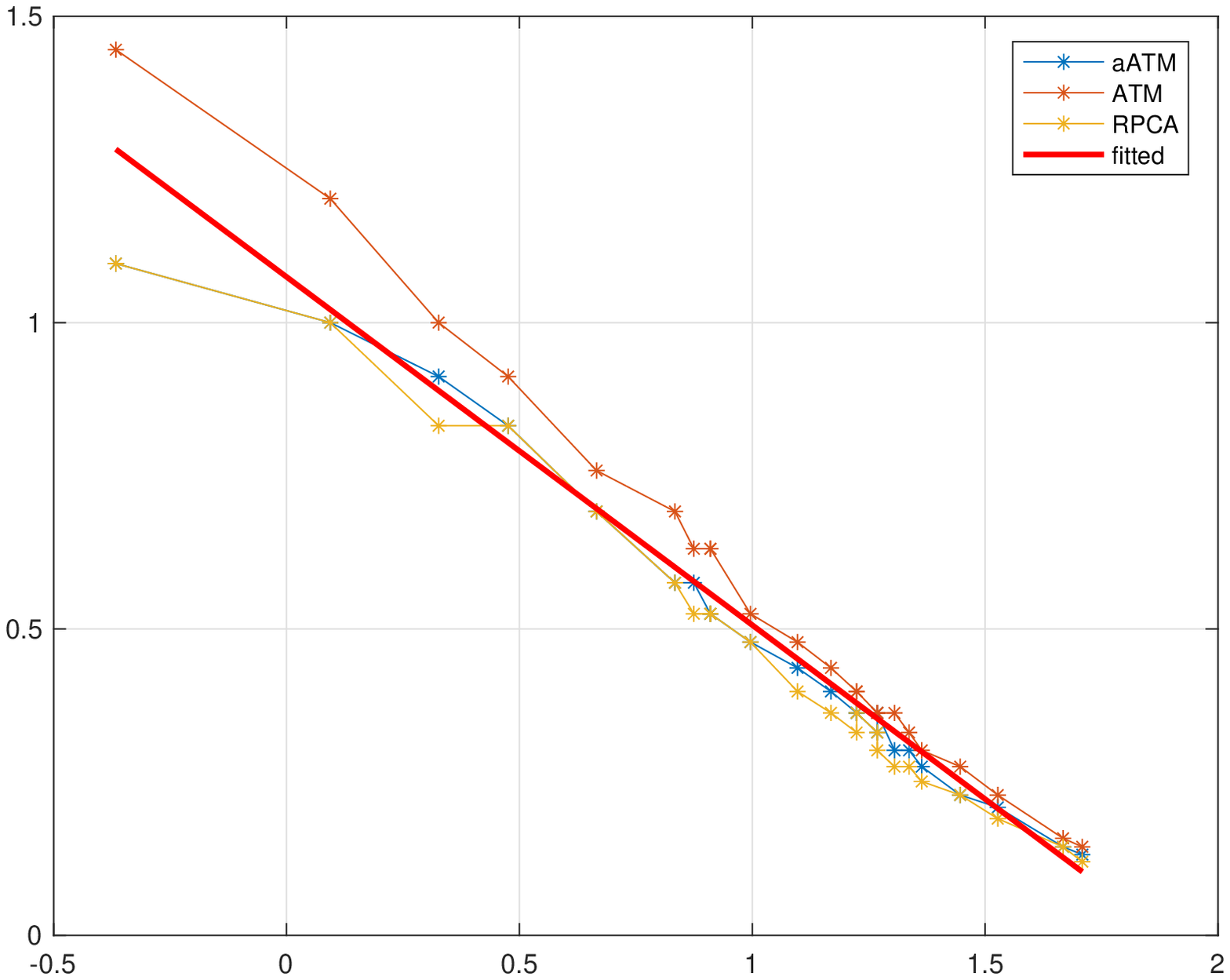}
		\caption{$\sqrt d\lambda^*$ v.s $n=2 \sim 250$. }
		\label{fig:bestlambdavsn}
	\end{center}
\end{figure}

\begin{rem}
	$\hat\lambda^*$ is estimated for all methods. However, it is possible to derive the estimator for individual method. ATM may be disadvantaged as the fit is not as good as others. $\hat\lambda^*$ should be lower bounded by the minimum value of $\lambda$ in Lemma \ref{lem:extremalambda} so 
	\[
	\hat\lambda^* = \max\{\frac{-0.5682\log(\log(n))+1.0747}{\sqrt d},\frac1{\sqrt{dn}}\}
	\]
	prevents $\lambda^*$ from being too small. Finally this $\hat\lambda^*$ is empirical and approximate with no model assumption. More sophisticated regression methods are possible.
\end{rem}

\subsection{Performance evaluation on simulations on multiple images}
Now we are ready for more comprehensive tests. One last question is how these evaluations hold across different $n$ (image sequence length) and different scenes? To this end, we picked 3 other images from Inria data set, {\tt chicago1}, {\tt kitsap1} and {\tt vienna1}, and ran the same randomised trials with $n=5, 7, 10, 12, 15, 20$, each with 50 repeats. The ground truth images are displayed in Fig. \ref{fig:simacrossimgs}. In this experiment, we bring in the state-of-the-art deep learning methods  \cite{Sarukkai_2020_WACV} (called  STGAN+Resnet and STGAN+Unet) and \cite{zhengSingleImageCloud2021} (called UNET and UNET+GAN) for a thorough comparison. 
STGAN provides two variants using Resnet and Unet backbone networks. UNET separates cloud and ground only and UNET+GAN uses GAN to fill thick cloud covered areas. 
The training of these deep learning models strictly followed the procedures in their code base repository, and were optimised for best performance as per instructions.  
For our models and RPCA, $\lambda$ was automatically determined by \eqref{e:bestlambda} for different $n$. 

Fig. \ref{fig:simacross} visaulises all $r$ values in one place, where the height of the bars are the means and error bars on top show the standard deviation calculated from multiple trials. The results of the same methods are grouped together with different coloured bars showing the results for different $n$. The overall impression is that deep learning methods are not as good although they have quite stable performance across different $n$ values. They may have some advantages when $n$ is small, say $n=5$. STGAN is better than RPCA in {\tt kitsap1} although no match for aATM and ATM when $n > 7$. Deep learning methods have large performance variations across different scenes, while others are rather consistent. All sparse models have better $r$ values when $n$ grows larger. This suggests that a strategy to boost performance is to increase the sampling frequency moderately. It makes perfect sense as more images provide more information for the missing pixels covered by clouds, and it is more likely that some areas covered in one image are not covered in another. The rank minimisation in aATM/ATM/RPCA is designed to fully utilise this. The unencessity of MC is once again verified in this test. The add-on value of MC is only observable for ATM when $n$ is small, i.e. $n<15$. 

The above observations provide us a clear clue to the questions raised at the beginning of this paper and reflect our motivation. Deep learning methods in general have lower fidelity (higher $r$ values). This uncertainty poses many questions for subsequent applications. They may have good performance on some specific scenes, for example, pure nature scene like {\tt kitsap1}. However, it is not clear how GAN's distribution transformation works. The large performance variation reveals their problems in dealing with different situations. While our models do not have these issues and interpretable in terms of their working mechanism.

\begin{figure}[htbp]
	\begin{center}
		\includegraphics[width=0.45\linewidth]{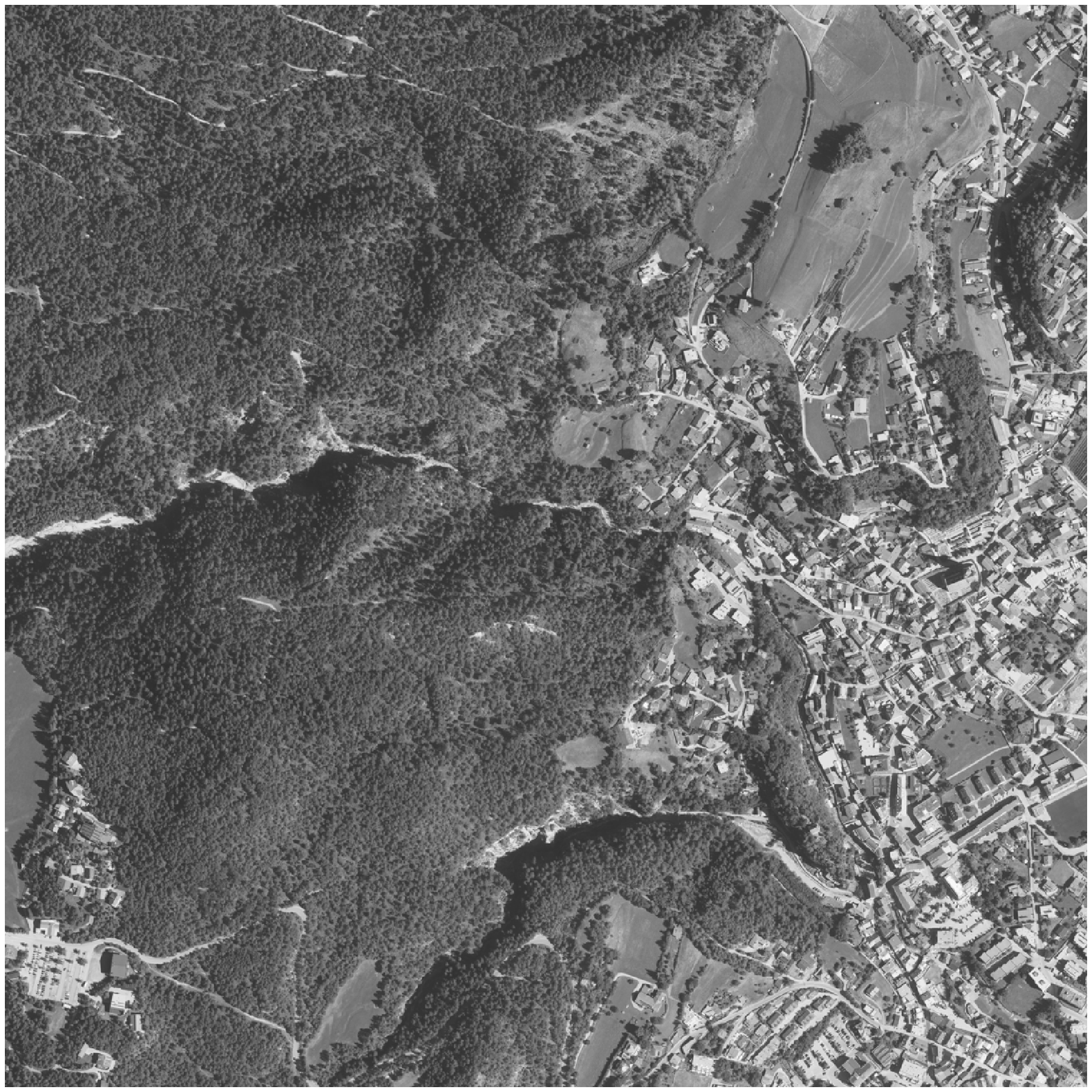}
		\includegraphics[width=0.45\linewidth]{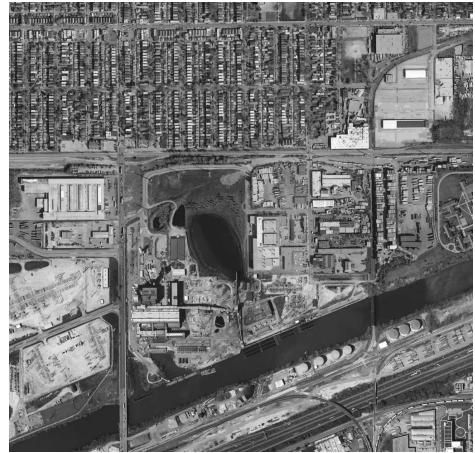} \\
		~\includegraphics[width=0.45\linewidth]{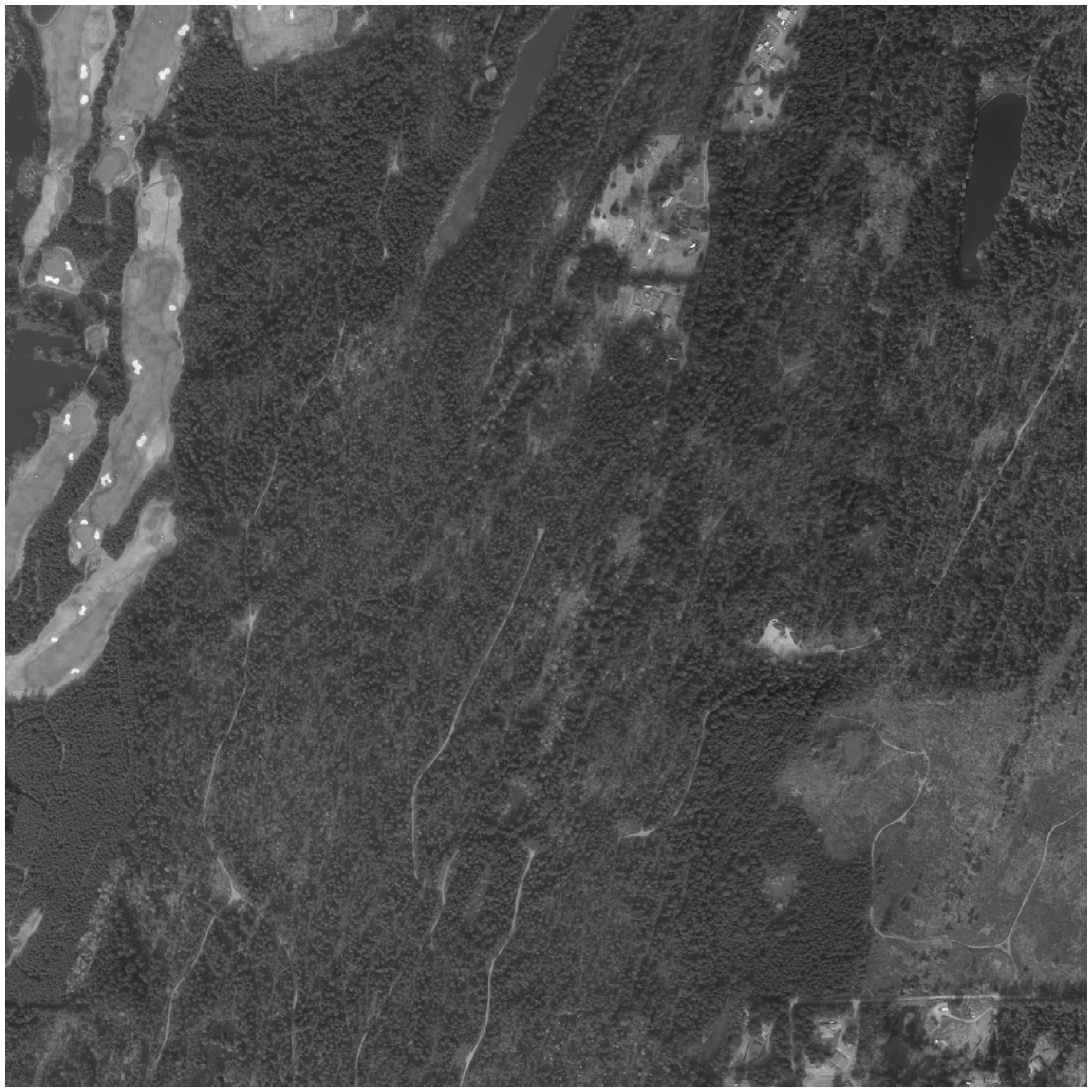}
		\includegraphics[width=0.45\linewidth]{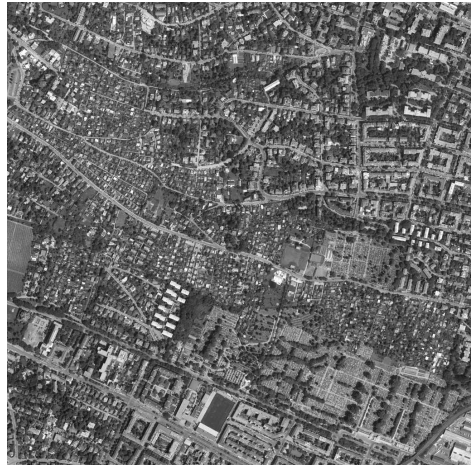}
		\caption{Four base images for randomised trials. From top to bottom and left to right:  {\tt tyrol-w1}, {\tt chicago1}, {\tt kitsap1} and {\tt vienna1}. }
		\label{fig:simacrossimgs}
	\end{center}
\end{figure}

\begin{figure*}[htbp]
	\begin{center}
		\includegraphics[width=0.45\linewidth]{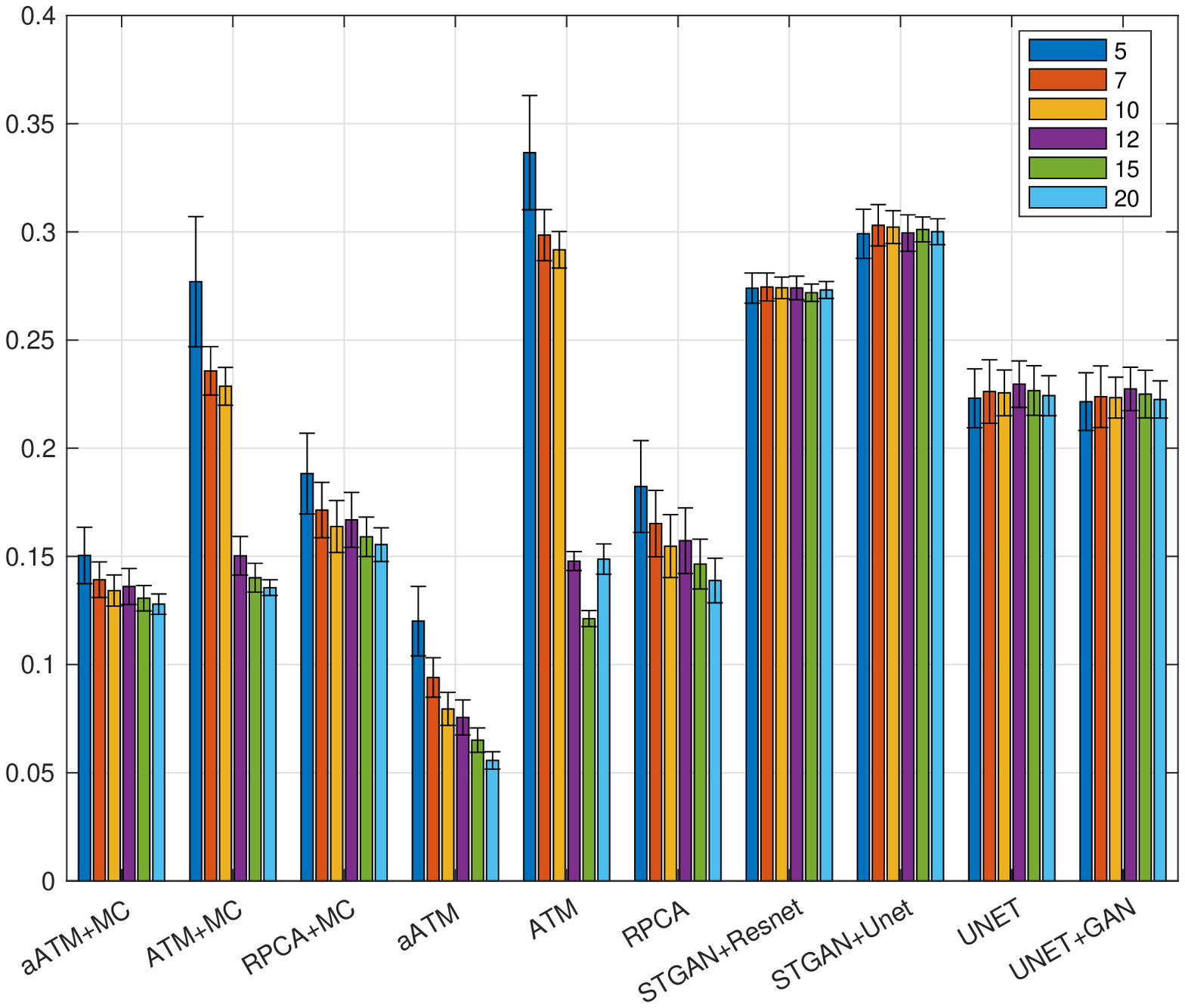}
		\includegraphics[width=0.45\linewidth]{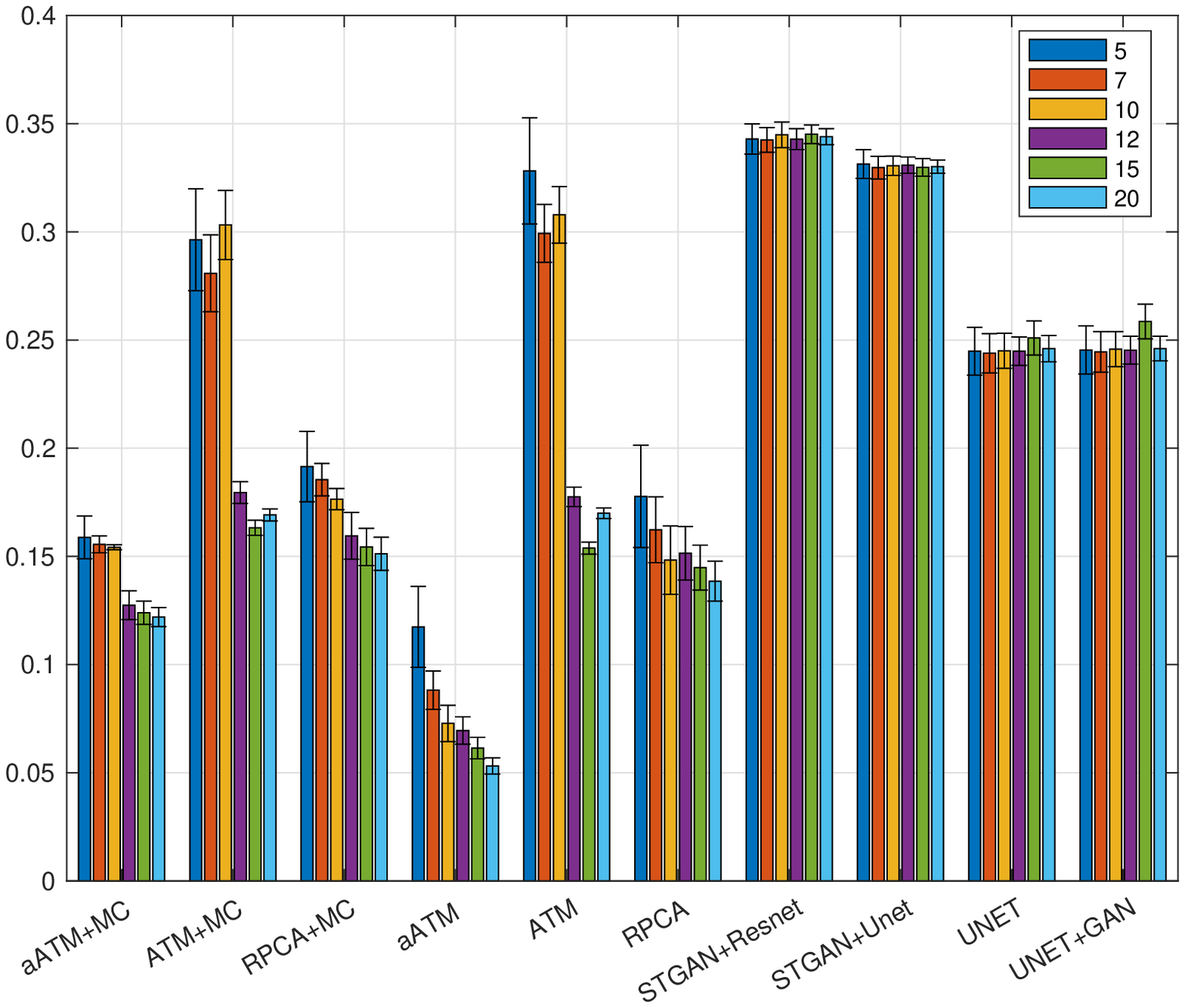}\\
		~\includegraphics[width=0.45\linewidth]{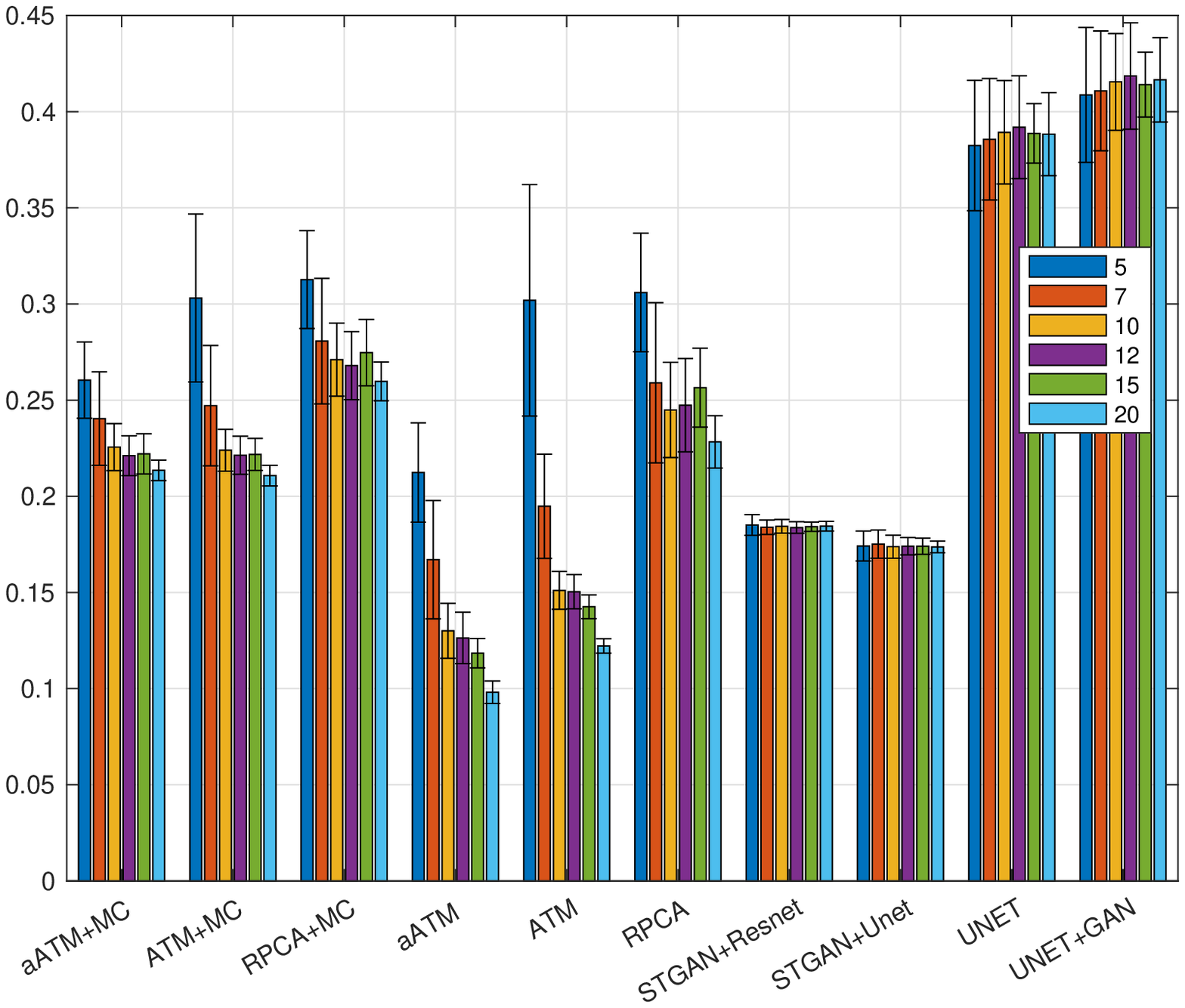}
		\includegraphics[width=0.45\linewidth]{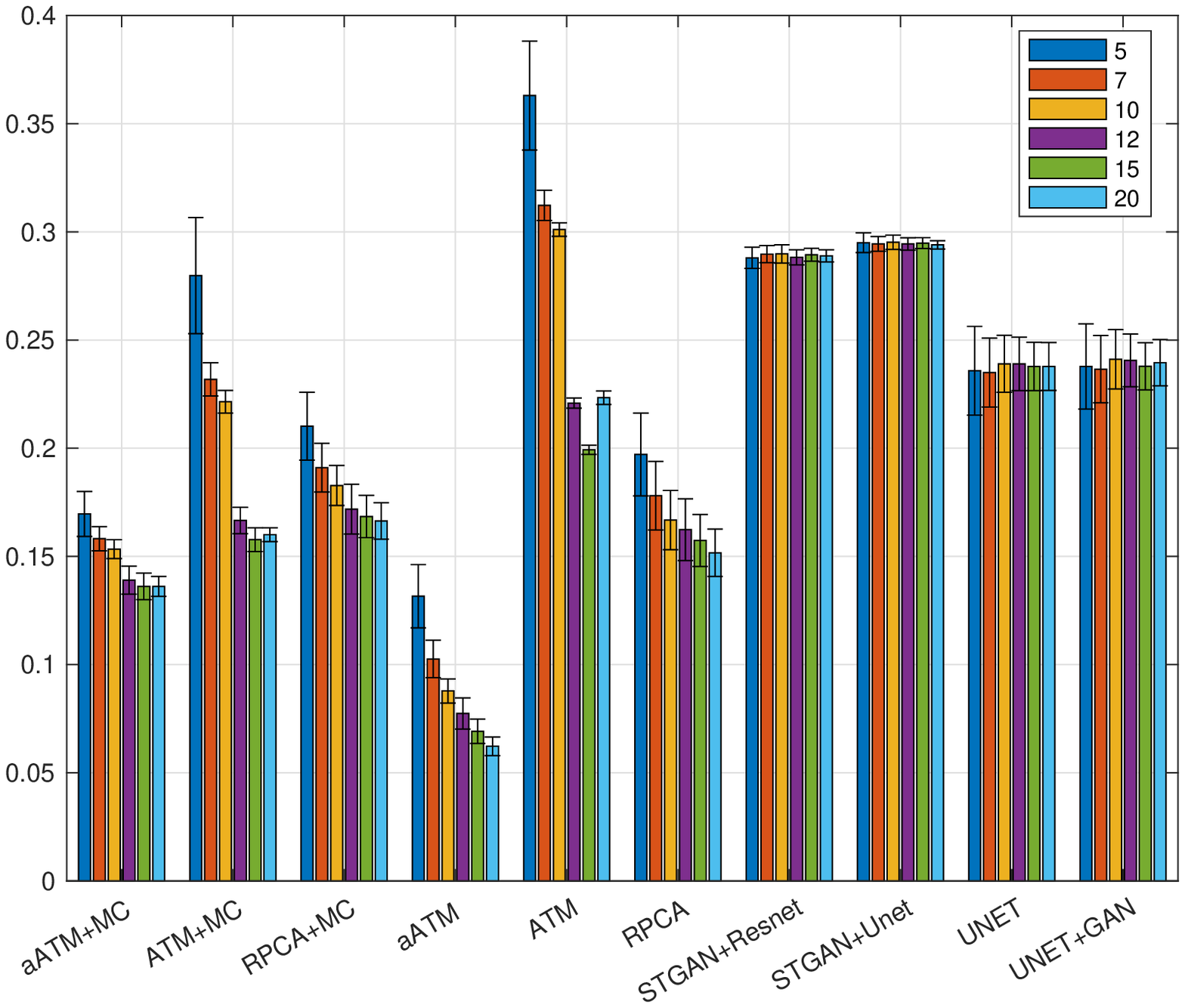}
		\caption{Randomised trials for various base images (in the order shown in Fig. \ref{fig:simacrossimgs}) and sequence length $n$. The legend shows $n$.   }
		\label{fig:simacross}
	\end{center}
\end{figure*}

\section{Conclusions}\label{sec:conclusion}
In this paper, we introduced atmosphere scattering model into cloud removal modelling process and proposed two ATM models as superior alternatives to RPCA based model. Furthermore, we proposed a method to simulate controllable cloud cover scenes. This semi-realistic simulation enables detailed study of various cloud removal methods, and provides valuable insights to several aspects of the algorithms as large scale randomised trial and quantitative analysis become possible. Examining the methods by using this powerful experimental tool, we saw clearly that the proposed aATM outperforms not only RPCA model,  the state-of-the-art in this category of non-deep-learning based cloud removal methods, but also latest deep learning models constructed on large scale backbone networks, by quite a large margin. There were many interesting findings in this process, for example the zoning of the regularisation parameter $\lambda$, computational cost pattern across zoning, automated regularisation parameter determination and so on. These may be out of the question without the assistance of the simulation. We envisage a robust development of the cloud removal algorithm under this framework in near future. 






%
%
%
\bibliographystyle{plain}
\bibliography{ref}

\begin{thebibliography}{10}
\expandafter\ifx\csname url\endcsname\relax
  \def\url#1{\texttt{#1}}\fi
\expandafter\ifx\csname urlprefix\endcsname\relax\def\urlprefix{URL }\fi
\expandafter\ifx\csname href\endcsname\relax
  \def\href#1#2{#2} \def\path#1{#1}\fi

\bibitem{zhengSingleImageCloud2021}
J.~Zheng, X.-Y. Liu, X.~Wang,
  \href{https://ieeexplore.ieee.org/document/9224941/}{Single {Image} {Cloud}
  {Removal} {Using} {U}-{Net} and {Generative} {Adversarial} {Networks}}, IEEE
  Transactions on Geoscience and Remote Sensing 59~(8) (2021) 6371--6385.
\newblock \href {https://doi.org/10.1109/TGRS.2020.3027819}
  {\path{doi:10.1109/TGRS.2020.3027819}}.
\newline\urlprefix\url{https://ieeexplore.ieee.org/document/9224941/}

\bibitem{Sarukkai_2020_WACV}
V.~Sarukkai, A.~Jain, B.~Uzkent, S.~Ermon, Cloud removal from satellite images
  using spatiotemporal generator networks, in: Proceedings of the {IEEE}/{CVF}
  winter conference on applications of computer vision ({WACV}), 2020.

\bibitem{guoCloudFilteringLandsat2016b}
Y.~Guo, F.~Li, P.~Caccetta, D.~Devereux, M.~Berman,
  \href{http://ieeexplore.ieee.org/document/7730888/}{Cloud filtering for
  {Landsat} {TM} satellite images using multiple temporal mosaicing}, in: 2016
  {IEEE} {International} {Geoscience} and {Remote} {Sensing} {Symposium}
  ({IGARSS}), IEEE, Beijing, China, 2016, pp. 7240--7243.
\newblock \href {https://doi.org/10.1109/IGARSS.2016.7730888}
  {\path{doi:10.1109/IGARSS.2016.7730888}}.
\newline\urlprefix\url{http://ieeexplore.ieee.org/document/7730888/}

\bibitem{guoMultipleTemporalMosaicing2017b}
Y.~Guo, F.~Li, P.~Caccetta, D.~Devereux,
  \href{http://remotesensing.spiedigitallibrary.org/article.aspx?doi=10.1117/1.JRS.11.015021}{Multiple
  temporal mosaicing for {Landsat} satellite images}, Journal of Applied Remote
  Sensing 11~(1) (2017) 015021.
\newblock \href {https://doi.org/10.1117/1.JRS.11.015021}
  {\path{doi:10.1117/1.JRS.11.015021}}.
\newline\urlprefix\url{http://remotesensing.spiedigitallibrary.org/article.aspx?doi=10.1117/1.JRS.11.015021}

\bibitem{wenTwoPassRobustComponent2018}
F.~Wen, Y.~Zhang, Z.~Gao, X.~Ling,
  \href{https://ieeexplore.ieee.org/document/8355668/}{Two-{Pass} {Robust}
  {Component} {Analysis} for {Cloud} {Removal} in {Satellite} {Image}
  {Sequence}}, IEEE Geoscience and Remote Sensing Letters 15~(7) (2018)
  1090--1094.
\newblock \href {https://doi.org/10.1109/LGRS.2018.2829028}
  {\path{doi:10.1109/LGRS.2018.2829028}}.
\newline\urlprefix\url{https://ieeexplore.ieee.org/document/8355668/}

\bibitem{zhangCoarsetoFineFrameworkCloud2019}
Y.~Zhang, F.~Wen, Z.~Gao, X.~Ling,
  \href{https://ieeexplore.ieee.org/document/8675771/}{A {Coarse}-to-{Fine}
  {Framework} for {Cloud} {Removal} in {Remote} {Sensing} {Image} {Sequence}},
  IEEE Transactions on Geoscience and Remote Sensing 57~(8) (2019) 5963--5974.
\newblock \href {https://doi.org/10.1109/TGRS.2019.2903594}
  {\path{doi:10.1109/TGRS.2019.2903594}}.
\newline\urlprefix\url{https://ieeexplore.ieee.org/document/8675771/}

\bibitem{CandesLiMaWright2010}
E.~J. Candès, X.~Li, Y.~Ma, J.~Wright,
  \href{http://doi.acm.org/10.1145/1970392.1970395}{Robust principal component
  analysis?}, Journal of the ACM 58~(3) (2011) 11:1--11:37, tex.acmid: 1970395
  tex.address: New York, NY, USA tex.articleno: 11 tex.issue\_date: May 2011
  tex.numpages: 37 tex.publisher: ACM.
\newblock \href {https://doi.org/10.1145/1970392.1970395}
  {\path{doi:10.1145/1970392.1970395}}.
\newline\urlprefix\url{http://doi.acm.org/10.1145/1970392.1970395}

\bibitem{delatorreRobustPrincipalComponent2001}
F.~De~la Torre, M.~Black,
  \href{http://ieeexplore.ieee.org/document/937541/}{Robust principal component
  analysis for computer vision}, in: Proceedings {Eighth} {IEEE}
  {International} {Conference} on {Computer} {Vision}. {ICCV} 2001, Vol.~1,
  IEEE Comput. Soc, Vancouver, BC, Canada, 2001, pp. 362--369.
\newblock \href {https://doi.org/10.1109/ICCV.2001.937541}
  {\path{doi:10.1109/ICCV.2001.937541}}.
\newline\urlprefix\url{http://ieeexplore.ieee.org/document/937541/}

\bibitem{GaoKwanGuo2009}
J.~Gao, P.~W. Kwan, Y.~Guo, Robust multivariate {L1} principal component
  analysis and dimensionality reduction, Neurocomputing 72 (2009) 1242--1249.

\bibitem{narasimhanVisionAtmosphere2002}
S.~G. Narasimhan, S.~K. Nayar,
  \href{https://doi.org/10.1023/A:1016328200723}{Vision and the {Atmosphere}},
  International Journal of Computer Vision 48~(3) (2002) 233--254.
\newblock \href {https://doi.org/10.1023/A:1016328200723}
  {\path{doi:10.1023/A:1016328200723}}.
\newline\urlprefix\url{https://doi.org/10.1023/A:1016328200723}

\bibitem{BoydVandenberghe2004}
S.~Boyd, L.~Vandenberghe, Convex optimization, Cambridge University Press,
  2004, tex.owner: guo020 tex.timestamp: 2014.07.24.

\bibitem{PongTsengJiYe2010}
T.~K. Pong, P.~Tseng, S.~Ji, J.~Ye, Trace norm regularization:
  {Reformulations}, algorithms, and multi-task learning, SIAM Journal on
  Optimization 20~(6) (2010) 3465--3489, tex.owner: jbgao tex.timestamp:
  2010.10.20.

\bibitem{caiSingularValueThresholding2010}
J.-F. Cai, E.~J. Candès, Z.~Shen,
  \href{https://epubs.siam.org/doi/10.1137/080738970}{A {Singular} {Value}
  {Thresholding} {Algorithm} for {Matrix} {Completion}}, SIAM Journal on
  Optimization 20~(4) (2010) 1956--1982.
\newblock \href {https://doi.org/10.1137/080738970}
  {\path{doi:10.1137/080738970}}.
\newline\urlprefix\url{https://epubs.siam.org/doi/10.1137/080738970}

\bibitem{Nesterov2003}
Y.~Nesterov, Introductory lectures on convex optimization: {A} basic course,
  Vol.~87 of Applied optimization, Kluwer Academic Publishers, 2003, tex.owner:
  guo020 tex.timestamp: 2014.07.24.

\bibitem{LiuJiYe2009a}
J.~Liu, S.~Ji, J.~Ye, {SLEP}: {Sparse} learning with efficient projection,
  Tech. rep., Arizona State University, tex.owner: jbgao tex.timestamp:
  2010.10.20 (2009).

\bibitem{EldarKutyniok2012}
Y.~C. Eldar, G.~Kutyniok (Eds.), Compressed sensing theory and applications,
  Cambridge University Press, 2012, tex.owner: guo020 tex.timestamp:
  2015.07.07.

\bibitem{ChenHeYeYuan2014}
C.~Chen, B.~He, Y.~Ye, X.~Yuan,
  \href{http://dx.doi.org/10.1007/s10107-014-0826-5}{The direct extension of
  {ADMM} for multi-block convex minimization problems is not necessarily
  convergent}, Mathematical Programming (2014) 1--23Tex.publisher: Springer
  Berlin Heidelberg.
\newblock \href {https://doi.org/10.1007/s10107-014-0826-5}
  {\path{doi:10.1007/s10107-014-0826-5}}.
\newline\urlprefix\url{http://dx.doi.org/10.1007/s10107-014-0826-5}

\bibitem{LigginsChongKadarAlfordVannicolaThomopoulos1997}
M.~Liggins, C.~Chong, I.~Kadar, M.~Alford, V.~Vannicola, S.~Thomopoulos,
  Distributed fusion architectures and algorithms for target tracking, Proc.
  IEEE 85~(1) (1997) 95--107.

\bibitem{jaccardDISTRIBUTIONFLORAALPINE1912}
P.~Jaccard,
  \href{https://onlinelibrary.wiley.com/doi/10.1111/j.1469-8137.1912.tb05611.x}{{THE}
  {DISTRIBUTION} {OF} {THE} {FLORA} {IN} {THE} {ALPINE} {ZONE}.1}, New
  Phytologist 11~(2) (1912) 37--50.
\newblock \href {https://doi.org/10.1111/j.1469-8137.1912.tb05611.x}
  {\path{doi:10.1111/j.1469-8137.1912.tb05611.x}}.
\newline\urlprefix\url{https://onlinelibrary.wiley.com/doi/10.1111/j.1469-8137.1912.tb05611.x}

\bibitem{dobashiUsingMetaballsModeling1999}
Y.~Dobashi, T.~Nishita, H.~Yamashita, T.~Okita,
  \href{http://link.springer.com/10.1007/s003710050193}{Using metaballs to
  modeling and animate clouds from satellite images}, The Visual Computer
  15~(9) (1999) 471--482.
\newblock \href {https://doi.org/10.1007/s003710050193}
  {\path{doi:10.1007/s003710050193}}.
\newline\urlprefix\url{http://link.springer.com/10.1007/s003710050193}

\bibitem{dobashiVisualSimulationClouds2017}
Y.~Dobashi, K.~Iwasaki, Y.~Yue, T.~Nishita,
  \href{https://linkinghub.elsevier.com/retrieve/pii/S2468502X17300013}{Visual
  simulation of clouds}, Visual Informatics 1~(1) (2017) 1--8.
\newblock \href {https://doi.org/10.1016/j.visinf.2017.01.001}
  {\path{doi:10.1016/j.visinf.2017.01.001}}.
\newline\urlprefix\url{https://linkinghub.elsevier.com/retrieve/pii/S2468502X17300013}

\bibitem{yuanModellingCumulusCloud2014}
C.~Yuan, X.~Liang, S.~Hao, Y.~Qi, Q.~Zhao,
  \href{https://onlinelibrary.wiley.com/doi/10.1111/cgf.12350}{Modelling
  {Cumulus} {Cloud} {Shape} from a {Single} {Image}: {Modelling} {Cumulus}
  {Cloud} {Shape} from a {Single} {Image}}, Computer Graphics Forum 33~(6)
  (2014) 288--297.
\newblock \href {https://doi.org/10.1111/cgf.12350}
  {\path{doi:10.1111/cgf.12350}}.
\newline\urlprefix\url{https://onlinelibrary.wiley.com/doi/10.1111/cgf.12350}

\bibitem{xingThreedimensionalParticleCloud2017}
y.~xing, j.~duan, y.~zhu, h.~wang,
  \href{https://www.spiedigitallibrary.org/conference-proceedings-of-spie/10605/2291982/Three-dimensional-particle-cloud-simulation-based-on-illumination-model/10.1117/12.2291982.full}{Three-dimensional
  particle cloud simulation based on illumination model}, in: Y.~Lv, J.~Su,
  W.~Gong, J.~Yang, W.~Bao, W.~Chen, Z.~Shi, J.~Fei, S.~Han, W.~Jin (Eds.),
  {LIDAR} {Imaging} {Detection} and {Target} {Recognition} 2017, SPIE,
  Changchun, China, 2017, p.~74.
\newblock \href {https://doi.org/10.1117/12.2291982}
  {\path{doi:10.1117/12.2291982}}.
\newline\urlprefix\url{https://www.spiedigitallibrary.org/conference-proceedings-of-spie/10605/2291982/Three-dimensional-particle-cloud-simulation-based-on-illumination-model/10.1117/12.2291982.full}

\bibitem{perlinImageSynthesizer1985}
K.~Perlin, \href{https://dl.acm.org/doi/10.1145/325165.325247}{An image
  synthesizer}, ACM SIGGRAPH Computer Graphics 19~(3) (1985) 287--296.
\newblock \href {https://doi.org/10.1145/325165.325247}
  {\path{doi:10.1145/325165.325247}}.
\newline\urlprefix\url{https://dl.acm.org/doi/10.1145/325165.325247}

\bibitem{perlinImprovingNoise2002}
K.~Perlin, \href{https://dl.acm.org/doi/10.1145/566654.566636}{Improving
  noise}, ACM Transactions on Graphics 21~(3) (2002) 681--682.
\newblock \href {https://doi.org/10.1145/566654.566636}
  {\path{doi:10.1145/566654.566636}}.
\newline\urlprefix\url{https://dl.acm.org/doi/10.1145/566654.566636}

\bibitem{maggiori2017dataset}
E.~Maggiori, Y.~Tarabalka, G.~Charpiat, P.~Alliez, Can semantic labeling
  methods generalize to any city? {The} inria aerial image labeling benchmark,
  in: {IEEE} international geoscience and remote sensing symposium ({IGARSS}),
  2017, tex.organization: IEEE.

\bibitem{azuhubiFunctionalAnalysis2010}
E.~S. Åžuhubi,
  \href{https://public.ebookcentral.proquest.com/choice/publicfullrecord.aspx?p=4712718}{Functional
  {Analysis}}, Springer Netherlands, Dordrecht, 2010, oCLC: 961064002.
\newline\urlprefix\url{https://public.ebookcentral.proquest.com/choice/publicfullrecord.aspx?p=4712718}

\bibitem{HastieRossetTibshiraniZhu2004}
T.~Hastie, S.~Rosset, R.~Tibshirani, J.~Zhu, The entire regularization path for
  the support vector machine, Journal of Machine Learning Research 5 (2004)
  1391--1415, tex.owner: guo020 tex.timestamp: 2014.07.24.

\bibitem{FriedmanHastieTibshirani2010a}
J.~H. Friedman, T.~Hastie, R.~Tibshirani,
  \href{http://www.jstatsoft.org/v33/i01}{Regularization paths for generalized
  linear models via coordinate descent}, Journal of Statistical Software 33~(1)
  (2010) 1--22, tex.accepted: 2009-12-15 tex.bibdate: 2009-12-15 tex.coden:
  JSSOBK tex.owner: guo020 tex.submitted: 2009-04-22 tex.timestamp: 2014.07.24.
\newline\urlprefix\url{http://www.jstatsoft.org/v33/i01}

\bibitem{TibshiraniTaylor2011}
R.~Tibshirani, J.~Taylor, The solution path of the generalized lasso, Annals of
  Statistics 39~(3) (2011) 1335--1371, tex.owner: guo020 tex.timestamp:
  2014.07.24.

\bibitem{taoRandomMatricesDistribution2010}
T.~Tao, V.~Vu, \href{http://link.springer.com/10.1007/s00039-010-0057-8}{Random
  {Matrices}: the {Distribution} of the {Smallest} {Singular} {Values}},
  Geometric and Functional Analysis 20~(1) (2010) 260--297.
\newblock \href {https://doi.org/10.1007/s00039-010-0057-8}
  {\path{doi:10.1007/s00039-010-0057-8}}.
\newline\urlprefix\url{http://link.springer.com/10.1007/s00039-010-0057-8}

\bibitem{baiLimitSmallestEigenvalue1993}
Z.~D. Bai, Y.~Q. Yin,
  \href{https://projecteuclid.org/journals/annals-of-probability/volume-21/issue-3/Limit-of-the-Smallest-Eigenvalue-of-a-Large-Dimensional-Sample/10.1214/aop/1176989118.full}{Limit
  of the {Smallest} {Eigenvalue} of a {Large} {Dimensional} {Sample}
  {Covariance} {Matrix}}, The Annals of Probability 21~(3) (Jul. 1993).
\newblock \href {https://doi.org/10.1214/aop/1176989118}
  {\path{doi:10.1214/aop/1176989118}}.
\newline\urlprefix\url{https://projecteuclid.org/journals/annals-of-probability/volume-21/issue-3/Limit-of-the-Smallest-Eigenvalue-of-a-Large-Dimensional-Sample/10.1214/aop/1176989118.full}

\bibitem{vershyninIntroductionNonasymptoticAnalysis2012}
R.~Vershynin,
  \href{https://www.cambridge.org/core/product/identifier/CBO9780511794308A012/type/book_part}{Introduction
  to the non-asymptotic analysis of random matrices}, in: Y.~C. Eldar,
  G.~Kutyniok (Eds.), Compressed {Sensing}, Cambridge University Press,
  Cambridge, 2012, pp. 210--268.
\newblock \href {https://doi.org/10.1017/CBO9780511794308.006}
  {\path{doi:10.1017/CBO9780511794308.006}}.
\newline\urlprefix\url{https://www.cambridge.org/core/product/identifier/CBO9780511794308A012/type/book_part}

\end{thebibliography}

\section*{Analysis of regularisation parameter $\lambda$}\label{sec:analysis}
In this section, we focus on the theoretic analysis on the regularisation parameter $\lambda$ in the models, in particular its valid range.  
We first have the following minimum $\lambda$ value lemma. 
\begin{lem} \label{lem:extremalambda}
	For any given data in $D\in\Real^{d\times n}$ assuming $d>n$, the minimum and maximum value for $\lambda$ in ATM, aATM and RPCA model is $\frac1{\sqrt{dn}}$ and $\sqrt{n}$ respectively. The extremum is in the sense of bound for the models to generate non-trivial solutions. 
\end{lem}
\begin{proof}
	The optimality condition of both RPCA and ATM models requires 
	\[
	\partial{\|L^*\|_*}+\lambda\partial{\|C^*\|_1} = 0
	\]
	leading to 
	\be\label{e:optcond}
	UV^\top + W + \lambda sign(C^*) = 0 
	\ee
	where superscribed star $^{*}$ stands for the optimal value, $L^* = U\Sigma V^\top$ is the skinny SVD of $L^*$ (i.e. $U\in\Real^{d\times l}=[u_1,\ldots,u_l]$, $V\in\Real^{n\times l}=[v_1,\ldots,v_l]$, and $\Sigma=I_l$, identity matrix of size $l\times l$) and $W\in\Real^{d\times n}$ is any matrix satisfying 
	\[
	\|W\|_2 \le 1,\ U^\top W = 0,\ WV=0. 
	\]
	It is easy to see that $W=U_{\perp}EV_{\perp}^\top$, where $A_{\perp}$ is the complementary components in ambient space that is orthogonal to matrix $A$, and $E$ is a diagonal matrix with all element $e_j\in[0,1]$ to satisfy $\|W\|_2 \le 1$. With the re-wrting Eq. \eqref{e:optcond}, we are seeking  
	\be\label{e:optcond1}
	\inf_\lambda\{\lambda \ge 0:\sum_{i=1}^lH_i + \sum_{j>l}^ne_jH_j = \lambda sign(-C^*)\}
	\ee
	where as $H_i=u_iv_i^\top$. $H_i$'s are orthogonal to each other and unitary in terms of Frobenius norm. Eq. \eqref{e:optcond1} shows that the subgradient of $\|L\|$ at $L^*$ is clamped by $\lambda$ regardless $C^*$, meaning the elements in the left hand side of \eqref{e:optcond1} have to be in the range of $[-\lambda, \lambda]$, a boxed condition. The largest norm within the $\lambda$ box is at one of its corners. Without loss of generality, we can choose the first orthant corner. According to Pethagorean, adding an orthogonal component to a vector will only increase the norm. Hence to allow as large as possible for $\lambda$, one can seek the vector with smallest norm, which reflected to the situation in \eqref{e:optcond1} is to let $l=1$ and set $\forall j\ e_j=0$. It is equivalent to choose $U$ and $V$ to be vectors of all $1/\sqrt{d}$ and all $1/\sqrt{n}$ respectively and let $W=0$. In this case, $ UV^\top + W=\frac1{\sqrt{dn}}\mat 1$ where $\mat 1$ is matrix of all one's with compatible dimensions. This is the smallest norm the subgradient of $L$ can fit in the $\lambda$ box. Therefore, the infimum in \eqref{e:optcond1}, i.e. what is required in this lemma is  $\frac1{\sqrt{dn}}$.
	
	Similarly the maximum is 
	\be\label{e:optcond2}
	\sup_\lambda\{\lambda \ge 0:\sum_{i=1}^lH_i + \sum_{j>l}^ne_jH_j = \lambda sign(-C^*)\}. 
	\ee
	The only difference is that one has to consider all possibilities, i.e. the maximum of the norm. Therefore, \eqref{e:optcond2} is equivalent to 
	\begin{align*}
	\max\{\|\sum_{i=1}^lH_i + \sum_{j>l}^ne_jH_j \|_F\}=\sqrt{\tr{V_nU_n^\top U_nV_n^\top}} =\sqrt{n}
	\end{align*}
	where $U_n=[u_1, \ldots,u_n]$ and $V_n$ likewise. 
	
\end{proof}
\begin{rem}
	From above, we can see that, when $\lambda<\frac1{\sqrt{dn}}$, the only allowed solution is to nullify the elements in $L$, in which case, $L^*=0$ and $C^*=D$. This is what we have seen in Fig. \ref{fig:sim7}, where when $\lambda$ is very small, $r$ value is 1 as $L=0$. aATM has the same result although it has another regularisation because  the infimum of $\lambda$ happens only when $\beta=\infty$, otherwise it would further reduce the value of $\lambda$. 
	
	This matches the purpose of regularisation.  When $\lambda$ is too small, the penalty to sparsity is next to null. Hence the sparse component is free. The sensible choice is of course to set the low rank component zeros, such that the objective is quite small, although this is a trivial solution. Similar logic for maximum value of $\lambda$. 
\end{rem}

Note that the extrema values of $\lambda$ deduced in Lemma \ref{lem:extremalambda} is for general cases, in other words, no specific conditions. The minimum value of $\lambda$ is very close to the recommended value of $\lambda$ in RPCA, while the maximum is rather loose, due to the generality. Actually we can have the following tighter upper bound of $\lambda$. 
\begin{lem} \label{lem:tightermaxlambda}
	For a given $D\in\Real^{d\times n}$ assuming $d>n$, the maximum value for $\lambda$ in ATM, aATM with $\beta=\infty$ and RPCA model, written as $\lambda_m$, is $\|UV^\top\|_\infty$ where $U$ and $V$ is from the skinny SVD of $D=U\Sigma V^\top$ and $\|A\|_\infty$ is the matrix infinity norm, i.e. the maximum absolution value of all its elements.
\end{lem}
\begin{proof}
	The proof is similar to that of Lemma \ref{lem:extremalambda}  except now we consider only one data set $D$, i.e. 
	\be\label{e:optcond3}
	\sup_\lambda\{\lambda \ge 0:\sum_{i=1}^lH_i \in [-\lambda,\lambda]\}, 
	\ee 
	as the sparse component $C^*$ is totally zero. $\sum_{i=1}^lH_i=UV^\top$ and all its elements are surely less than or equal to $\|UV^\top\|_\infty$. Therefore when $\lambda=\|UV^\top\|_\infty$, the condition in \eqref{e:optcond3} holds. This is equivalent to setting $\forall j,\ e_j=0$ as in \eqref{e:optcond2}. 
\end{proof}

The upper bound of $\lambda$ in Lemma \ref{lem:tightermaxlambda} is much better than that in Lemma \ref{lem:extremalambda}, especially when $d\gg n$ and $D\in[0,1]$. However it is possible to further quantify $\lambda=\|UV^\top\|_\infty$ without actual SVD. We give asymptotic results here of the upper bound of $\lambda=\|UV^\top\|_\infty$ and hence $\lambda_m $. To proceed, we need the following proposition to bound $\|UV^\top\|_\infty$. 
\begin{prop}\label{prop:inftynormconnection}
	For any matrix $A$ of size $d\times n$ ($d>n$)and its skinny SVD as $A= U\Sigma V^\top$, the following holds
	\[
	\|UV^\top\|_\infty \le \frac{\|A\|_\infty}{\sigma_{n}(A)}
	\]
	where $\sigma_{i}(A)$ is the $i$th largest singular value of $A$ and then $\sigma_{n}(A)$ is the smallest singular value of $A$.
\end{prop}
\begin{proof}
	Following the same way of thinking from previous lemmas, we see that $\|UV^\top\|_\infty$ can only bound the Frobenius norm of the matrix spanned by the same bases up to $\sqrt n$. In other words, for any matrix of size $d\times n$, if $\|X\|_F\ge\sqrt n$, then $\|X\|_\infty\ge\|UV^\top\|_\infty$. Also $\|\sum_i \sigma_iu_iv_i^\top\|_F^2=\sum_i \sigma_i^2$. Therefore, if $\sigma_i\ge1$ for all $i$, then $\|\sum_i \sigma_iu_iv_i^\top\|_F\ge\sqrt n$. Combining this observation with the fact that $\tilde A=\frac1{\sigma_n(A)} A = U\Sigma/\sigma_n V^\top$, we obtain the claim in this proposition, as we have $\|\tilde A\|_\infty>\|UV^\top\|_\infty$ since $\sigma_1(\tilde A) \ge \sigma_1(\tilde A)\ge \ldots \ge \sigma_n(\tilde A) =1$. 
\end{proof}
Now we treat the observed image matrix $D$ as a random matrix whose elements are {\it i.i.d} from uniform distribution from $[0,1]$ and hence we do not assume any further structure. We are then concerned with the smallest singular value $\sigma_n(D)$ of a non-central random matrix, i.e. the mean is non-zero. There is limited results on smallest singular values. The closest one is \cite{taoRandomMatricesDistribution2010}, which deals with centralised random matrix. Fortunately, we only need a lower bound on $\sigma_n(D)$. We use the following theorem from \cite{baiLimitSmallestEigenvalue1993}, which also appeared in \cite{vershyninIntroductionNonasymptoticAnalysis2012}. 

\begin{thm}[Bai-Yin’s law]\label{thm:bylaw}
	Let A be a $d\times n$ random matrix whose entries are independent copies of a random variable with zero mean, unit variance, and finite fourth moment. Suppose that the dimensions $d$ and $n$ grow to infinity while the aspect ratio $n/d$ converges to a constant in $[0, 1]$. Then
	\[
	\sigma_n(A) = \sqrt d - \sqrt n + o(\sqrt n),\ \sigma_1(A) = \sqrt d + \sqrt n + o(\sqrt n) 
	\] 
	almost surely.
\end{thm}
Note in Bai-Yin’s law there is no assumption on the distribution but centrality. We then write $D = A + 0.5E$ where $A$'s elements are i.i.d from uniform distribution from $[-0.5,0.5]$. We use the Courant-Fischer minimax characterisation of singular values to obtain the bound
\[
\sigma_k(A) = \max_{dim(\mathcal S) = k}\left\{ \min_{\|v\|_2=1,v\in \mathcal S} \|A v\|_2 \right\}
\]
where $\mathcal S$ is any subspace of $\Real^n$. This leads to 
\[
\sigma_n(A) =\min_{\|v\|_2=1,v\in \Real^n} \|A v\|_2
\]
as $\mathcal S$ is just $\Real^n$. We have the following lemma tailored to non-central uniform distribution. 
\begin{lem}
	Let A be a $d\times n$ random matrix whose entries are independent copies of a random  variable from uniform distribution with mean $\varepsilon$, unit variance. Suppose that the dimensions $d$ and $n$ grow to infinity while the aspect ratio $n/d$ converges to a constant in $[0, 1]$. Then
	\[
	\sigma_n(A) \ge \sqrt d - \sqrt n + o(\sqrt n)
	\] 
	almost surely.
\end{lem}\label{lem:lowerboundsmallestSV}
\begin{proof}
	Let $E$ be a all 1 matrix with compatible dimensions, then $A= \tilde A + \varepsilon E$ where $\tilde A$ satisfy conditions in \ref{thm:bylaw}. We have 
	\begin{align*}
	\sigma_n(A)^2 &=\min_{\|v\|_2=1,v\in \Real^n} \|(\tilde A + \varepsilon E) v\|_2^2\\
	&\ge  \min_{\|v_1\|_2=1,v_1\in \Real^n} \|\tilde Av_1\|_2^2 + \min_{\|v_2\|_2=1,v_2\in \Real^n} \|\varepsilon Ev_2\|_2^2 \\
	&+ \min_{\|v_3\|_2=1,v_3\in \Real^n} v_3^\top\tilde A^\top E v_3
	\end{align*}
	Note we use $v_j$ $j=1,2,3$ to highlight that these minimisations are separated and hence the above holds. $\tilde A^\top E$ in the third terms gives 
	\[
	\tilde A^\top E = ( s_1, s_2, \ldots, s_n )^\top 1_n^\top
	\]
	where $s_i = \sum_j^d a_{ij}$, $a_{ij}$ is the $ij$th element in $A$ and $1_n$ is the vector with all 1 with length $n$. As $a_{ij}$'s are from centralised population, under asymptotic condition, $s_i=0$ almost surely and hence the third term vanishes. The first two terms are the square of the smallest singular values of corresponding matrices, i.e. $\tilde A$ and $E$. Since $\sigma_i(E)=0$ for $i\not = 1$, we have the required inequality by using Theorem \ref{thm:bylaw}. 
\end{proof}

Combining Lemma \ref{lem:lowerboundsmallestSV} and Proposition \ref{prop:inftynormconnection}, we have the following corollary. 
\begin{cor}\label{cor:betterlambdamax}
	Assume the elements in data matrix $D$ of size $d\times n$ is uniformly distributed in $[0,1]$ and its size grows asymptoticly as in Lemma \ref{lem:lowerboundsmallestSV}. Given the conditions in Lemma \ref{lem:tightermaxlambda}, the maximum value for $\lambda$ is almost surely
	\be\label{e:lambdamax}
	\frac{2\sqrt 3}{\sqrt d - \sqrt n + o(\sqrt n)}
	\ee
\end{cor}
\begin{proof}
	It is simply the rescaling result of Lemma \ref{lem:lowerboundsmallestSV} by recognising the standard uniform distribution has variance $1/\sqrt{12}$ and also $\|D\|_\infty=1$. 
\end{proof}
\begin{rem}
	Although $\lambda_m$ in Corollary \ref{cor:betterlambdamax} is asymptotic result, as the images are quite large, say $2^{10}\times 2^{10}$ in our experiments, i.e. $d=2^{20}$, the bound of $\lambda_m$ is quite good. In practice, $n$ is relatively small, typically at the order of 10, $n\ll d$. Therefore we can further simplify \eqref{e:lambdamax} to 
	\[
	\lambda_m \approx \frac{2\sqrt 3}{\sqrt d}
	\]
	That is what we see from Fig. \ref{fig:sim7} and \ref{fig:sim15} that when $\lambda$ is too large, precisely larger than 0.0032 as shown in the figure, the sparse component is erased, i.e. $C=0$. In this case, our theory predicted $\lambda_m =  \frac{2\sqrt 3}{\sqrt{2^{20}} }=0.0034$, very close to our observation. When $C=0$, $L=D$ and hence the observed images with clouds. The variations we see from the figures are due to the simulated clouds. 
\end{rem}

\end{document}